\documentclass[authoryear]{elsarticle}

\usepackage{amsmath}
\usepackage{amssymb}
\usepackage{amsthm}
\usepackage{adjustbox}
\usepackage{tikz-cd}
\usepackage{mathtools}
\usepackage{marvosym}
\usepackage{listings}
\usepackage{booktabs}
\usepackage{longtable}
\usepackage{natbib}
\usepackage{algpseudocode}

\usepackage[polutonikogreek,english]{babel}
\usepackage[utf8x]{inputenx}

\newtheorem{Language}{Language}
\newtheorem{principle}[Language]{Principle}
\newtheorem{assumption}[Language]{Assumption}
\newtheorem{formalization}[Language]{Formalization}
\newtheorem{parametrization}[Language]{Parametrization}
\newtheorem{theorem}{Theorem}[section]
\newtheorem{lemma}[theorem]{Lemma}
\newtheorem{corollary}[theorem]{Corollary}
\newtheorem{proposition}[theorem]{Proposition}

\theoremstyle{definition} 
\newtheorem{definition}{Definition}[section]
\newtheorem{example}{Example}[section]
\newtheorem{notation}{Notation}[section]
\newtheorem{remark}{Remark}[section]

\newtheorem*{case}{Case}

\DeclareMathOperator{\MLN}{MLN}

\DeclareMathOperator{\Causal}{Causal}

\DeclareMathOperator{\CONST}{constraint}

\DeclareMathOperator{\causal}{Causal}
\DeclareMathOperator{\comp}{comp}

\DeclareMathOperator{\pa}{pa}

\DeclareMathOperator{\Do}{do}

\DeclareMathOperator{\effect}{effect}

\DeclareMathOperator{\graph}{graph}

\DeclareMathOperator{\Pa}{Pa}

\DeclareMathOperator{\error}{error}
\DeclareMathOperator{\Error}{Error}

\DeclareMathOperator{\LP}{LP}

\DeclareMathOperator{\explains}{explains}
\DeclareMathOperator{\constraint}{constraint}
\DeclareMathOperator{\sufficient}{sufficient}
\DeclareMathOperator{\explanatory}{explanatory}
\DeclareMathOperator{\EXP}{explanatory}
\DeclareMathOperator{\cause}{cause}
\DeclareMathOperator{\necessary}{necessary}

\usepackage{amsmath}
\usepackage{url}

\journal{International Journal of Approximate Reasoning}
\newcommand{\greek}[1]{{\selectlanguage{polutonikogreek}#1}}

\begin{document}

\begin{frontmatter}



\title{How Artificial Intelligence Leads to Knowledge Why: \\ An Inquiry Inspired by Aristotle’s \textit{Posterior Analytics}}  
\author[tuebingen_phil]{Guus Eelink}
\author[tuebingen_math]{Kilian R\"{u}ckschlo\ss}
\author[guds,ifi]{Felix Weitk\"amper}

\affiliation[tuebingen_phil]{
               organization=
               {\mbox{Universität Tübingen}},
               addressline={\mbox{Bursagasse 1}},
               city={Tübingen},
               postcode={72070},
               country={Germany}}

\affiliation[tuebingen_math]{
               organization={\mbox{Universität Tübingen}},
               addressline={\mbox{Auf der Morgenstelle 10 (C-Bau)}},
               city={Tübingen},
               postcode={72076},
               country={Germany}}

\affiliation[guds]{
               organization={\mbox{German University of Digital Science}},
               addressline={\mbox{Marlene-Dietrich-Allee 14}},
               city={Potsdam},
               postcode={14482},
               country={Germany}}

 \affiliation[ifi]{
               organization={\mbox{Fakultät für Informatik der LMU München}},
               addressline={\mbox{Oettingenstr. 67}},
               city={München},
               postcode={80538},
               country={Germany}}




\begin{abstract}
Bayesian networks and causal models provide frameworks for reasoning about external interventions, enabling tasks that go beyond what probability distributions alone can support. Although these formalisms are often informally described as encoding causal knowledge, there is a lack of a formal theory that characterizes the kind of knowledge required to predict the effects of such interventions. This work introduces the theoretical framework of \emph{causal systems} to implement Aristotle's distinction between knowledge-\emph{that} and knowledge-\emph{why} within the setting of artificial intelligence. By interpreting existing AI technologies as causal systems, it examines the corresponding forms of knowledge they embody. Finally, it argues that predicting the effects of external interventions is possible only with knowledge-\emph{why}, offering a more precise account of the assumptions underlying this capacity.
\end{abstract}




\end{frontmatter}





\section{Introduction}
\label{sec: Introduction}

Causality has been a central topic in philosophy for over two thousand years. More recently, \cite{Causality} brought it into the mainstream of artificial intelligence research. His key insight is that causal knowledge goes beyond descriptive knowledge by supporting queries about the effects of external interventions.

\begin{example}
    Let $h_1$ and $h_2$ be neighboring houses. A fire may break out in either. House~$h_1$ typically catches fire if House~$h_2$ is burning, and vice versa.

    Observing that House~$h_1$ is not burning, one might infer that $h_2$ is not burning either. However, this conclusion is not justified if fires are actively prevented in House~$h_1$, for instance by a sprinkler system. Predicting such intervention effects requires causal---rather than merely descriptive---knowledge.
    \label{example - houses introduction}
\end{example}



Since evaluating the effects of possible actions is a core motivation for modeling, Pearl's theory has been widely adopted across diverse domains \citep{ArifM22, GaoZWFHL24, WuPLZSLQLG24}. 

While causality in artificial intelligence is primarily studied through queries about interventions and counterfactuals, in philosophy it is more often examined in terms of explaining why things happen.

\begin{example}
    In Example~\ref{example - houses introduction}, concluding from a fire in House~$h_1$ \textit{that} a fire broke out in House~$h_1$ or $h_2$ runs against the direction of cause and effect. Such reasoning fails to provide a causal explanation or knowledge \textit{why} a fire occurred.

    By contrast, concluding from a fire breaking out in House~$h_1$ to a fire in House~$h_2$ respects the causal direction. If the origin of the initial fire lies beyond the scope of the model, this yields causal knowledge \textit{why} House~$h_2$ burns.
    \label{example - explanations}
\end{example}

\cite{Causality} develops his theory entirely within his own formalisms -- Bayesian networks and structural causal models -- which he claims capture causal knowledge. Yet, from a philosophical perspective, it remains unclear in what sense these formalisms are genuinely causal.

This ambiguity becomes problematic when Pearl’s ideas are applied to other frameworks, such as probabilistic logic programming~\citep{BaralH07, cplogic, ruckschloss2021exploiting}. Pearl derives causal knowledge by solving systems of defining equations -- an approach that primarily addresses the non-Boolean case. In the cyclic Boolean setting, logic programming provides various semantics for such equations~\citep{kowalski,StableModelSemantics,Clark}. Without a clear formalization of causal knowledge and explanation -- particularly the kind required to predict intervention effects -- such transfers risk misinterpretation and inconsistency.

To address this, the present work connects Pearl's theory to philosophical notions of causal explanation and causal knowledge from epistemology.


\subsection{The Notion of Knowledge in Aristotle's \emph{Posterior Analytics}}
\label{sec: The Notion of Knowledge in Aristotle's Posterior Analytics}

In the \textit{Posterior Analytics} Aristotle sets out his theory of scientific knowledge~(\greek{ἐπιστήμη}).\footnote{\citet{Barnes1995} translates ``\greek{ἐπιστήμη}'' with ``understanding''.} His account provides three key insights into the logic of causal explanations, which lie at the basis of this contribution.

\subsubsection{Knowledge by Demonstration}
\label{subsubsec: Knowledge by demonstration}

In an Aristotelian science facts are explained by way of a so-called \textit{demonstration} (\greek{ἀπόδειξις}), which is a type of deduction displaying the scientific explanation of a fact by deducing it from causally fundamental premises. Within Aristotle's logical theory, demonstrations are a proper subset of syllogisms, the valid deductions which Aristotle characterizes and classifies in his \textit{Prior Analytics}.\footnote{\citet{Barnes1995} also has a translation of the \textit{Prior Analytics}.} A demonstration is therefore a valid deduction which follows the causal order.
\footnote{Cf. \textit{Posterior Analytics} 1.2, 71b16-72b4, translated by \citet{Barnes1995}, pp. 115-116.}
Hence, Aristotle provides the following principles: 

\begin{principle}[Consistency with Deduction]
    Causal explanations or \emph{demonstrations} are an instance of logical deduction, i.e., \emph{syllogisms}.
    \label{principle - consistency with logical reasoning}
\end{principle}

\begin{principle}[Directionality]
    Causal explanations or \emph{demonstrations} proceed exclusively from causes to effects.
    \label{principle - directionality}
\end{principle}

\subsubsection{Indemonstrable knowledge}
\label{subsubsec: Indemonstrable knowledge}

Aristotle holds that demonstrations cannot be infinite or cyclic arguments.

\begin{principle}[Non-Circularity]
    Causal explanations must not be cyclic or result in infinite regress.
    \label{principle - non-circularity}
\end{principle}

\begin{example}
    In Example \ref{example - houses introduction}, concluding from a fire in House $h_1$ to a fire in House $h_2$, and back to a fire in House $h_1$, does not yield a causal explanation.
    \label{example - cyclic explanation}
\end{example}

Therefore, demonstrations must have as their starting-points premises which themselves cannot be further demonstrated.\footnote{Cf. \textit{Posterior Analytics} 1.3, 72b5-73a20, translated by \citet{Barnes1995}, pp. 117-118.} In this light Aristotle argues that there is also \textit{indemonstrable} knowledge. It is obtained from the essences of things, which in Aristotle's ontology are the fundamental constituents of reality to which all causal explanations in science should be traced back. To acquire indemonstrable knowledge a scientist needs to possess insight into these essences, which is called \textit{nous} (\greek{νοῦς}).\footnote{Aristotle discusses \textit{nous} in the last chapter of the \textit{Posterior Analytics}, 2.19, translated by \citet{Barnes1995}, pp. 165-166, who renders ``\greek{νοῦς}'' as ``comprehension''.} 
 As the meaning of \textit{nous} in artificial intelligence is unclear, this work avoids \textit{nous} and relies on Aristotle's subordination of sciences. 

\subsubsection{Knowledge-\emph{that}, knowledge-\emph{why} and the Subordination of Sciences}
\label{subsubsec: Knowledge that and knowledge why}

A third insight from Aristotle's theory of science is that the facts can be established even if one does not yet have scientific explanations of them. Aristotle allows for this by distinguishing between knowledge of the \textit{that} (\greek{ὅτι}) and knowledge of the \textit{why} (\greek{διότι}). The scientist first makes observations and collects data and thus acquires knowledge-\textit{that} of a set of facts without yet knowing the scientific explanations of those facts —  thus not yet having knowledge-\textit{why}. In order to acquire knowledge-\textit{why}, she must subsequently gain an understanding of the underlying essences and determine which facts follow directly from these essences, and are thus indemonstrable, and which facts can be demonstrated by way of those indemonstrable facts. On this basis she can then, in the final stage of her research, construct demonstrations and obtain knowledge-\textit{why}. 

As illustrated in Example \ref{example - explanations}, knowledge-\textit{that} may itself come with a kind of explanation which falls short of being scientific and therefore does not yield knowledge-\textit{why}.\footnote{Cf. \textit{Posterior Analytics} 1.13, 78a28-78b4, translated by \citet{Barnes1995}, pp. 127-128.} Such an explanation involved in knowledge-\textit{that} is deficient in that it does not follow the causal order of things. 

The distinction between knowledge-\textit{that} and knowledge-\textit{why} also plays a role in Aristotle's subordination of areas of scientific inquiry. Aristotle holds that certain areas of science are subordinated to others, which means that the premises used by the subordinate areas —  for instance, optics —  are explained by the superordinate areas —  for instance, geometry, in the case of optics.  The subordinate area of science can take as premises for its demonstrations the results of the superordinate area and on this basis explain the phenomena it is concerned with and thus yield knowledge-\textit{why}.\footnote{Cf. \textit{Posterior Analytics} 1.13, 79a10-16, translated by \citet{Barnes1995}, pp. 128-129.} 

This idea of a subordination of sciences is useful within the context of artificial intelligence. Instead of needing \textit{nous} to obtain knowledge of the premises of demonstrations, an intelligent system can obtain the premises from a superordinate science and on their basis construct demonstrations, thus yielding knowledge-\textit{why} within the given area of science. 

\begin{principle}[Causal Foundation]
     Knowledge-\emph{why} is obtained within an area of science. \textit{Causal explanations} or \emph{demonstrations} that yield knowledge-\emph{why} must originate from external premises \( \mathcal{E} \) that are further demonstrated in a superordinate area of science. 
     \label{principle - causal foundation}
\end{principle}

This work interprets formalisms from artificial intelligence as areas of science that are entirely grounded in superordinate areas of science and thus do not rely on \textit{nous}. It transfers the notions of \emph{demonstration} and knowledge-\emph{why} into this setting, and argues that this form of knowledge is embodied in Pearl’s formalisms, underlying their ability to support reasoning about interventions. To attain genuine scientific knowledge in the Aristotelean sense through this approach, however, one must gain insight into the essences of things—potentially through human--machine interaction.

\subsection{Outline of the Paper}

Section~\ref{sec: Knowledge in Deterministic Systems} explores causal reasoning in a deterministic Boolean setting:

Section~\ref{subsec: Preliminaries} introduces the logical theory of causality from \cite{Bochman}, which this work builds on. Section~\ref{subsec: Critique of Bochman's Logical Theory of Causality} critiques this formalism, highlighting issues with causal cycles. Section~\ref{subsec: Causal Systems} addresses these issues by proposing \emph{deterministic causal systems} as a general framework for reasoning about \mbox{knowledge-\textit{why}}. Section~\ref{subsec: Interpreting Artificial Intelligence Frameworks as Deterministic Causal Systems} analyzes the causal models of \cite{Causality} within this framework, Section~\ref{subsubsec: External Interventions in Causal Systems} extends the treatment of interventions, and Section~\ref{sec: The Constraint and Explanatory Content of Causal Reasoning} prepares the framework for probabilistic generalization.

Section~\ref{sec: Knowledge under Uncertainty} introduces uncertainty into this theory:

Section~\ref{subsec: Preliminaries for Uncertainty} presents LogLinear models~\mbox{\citep{MLN}}, Bayesian networks, and probabilistic causal models, along with Pearl's notion of intervention~\citep{Causality}. Section~\ref{subsec: Causal Systems: A Generic Representation of Causal Reasoning} introduces \emph{maximum entropy causal systems} for reasoning about knowledge-\textit{why} under uncertainty. Section~\ref{subsec: Interpreting Current Artificial Intelligence Technologies as Causal Systems} interprets the technologies introduced earlier within this framework, analyzing their knowledge content and generalizing the notion of intervention.

Finally, Section~\ref{sec: Conclusion} concludes the paper, and~\ref{section: Glossary} provides a glossary.

\subsubsection{Bochman’s Theory: A Starting Point in the Acyclic Case}
\label{subsec: Applying Aristotle's Notion of Knowledge in Artificial Intelligence}

\cite{Bochman} applies the idea that causal relations are typically expressed in the form of rules or laws. As noted in Chapter~1 of \mbox{\cite{Hulswit}}, this idea was first articulated by Descartes:

\begin{principle}[Causal Rules]
``...we can obtain knowledge of the rules or laws of nature, which are the secondary and particular causes...''~(Ren\'e Descartes: Principles of Philosophy II:37; translation by \cite{Miller2009})
\label{principle - causal rules}
\end{principle}   

\cite{Bochman} concludes that causal knowledge should be expressed by a \emph{causal theory}~$\Delta$, consisting of a set of causal rules of the form~$\phi \Rightarrow \psi$, where~$\phi$ and~$\psi$ are statements. The expression~$\phi \Rightarrow \psi$ is read as \mbox{``$\phi$ \emph{causes} $\psi$''}. Such a causal rule indicates that there exists a demonstration of~$\psi$ based on the premise~$\phi$. Consequently, \mbox{knowledge-\textit{why}}~$\phi$ gives rise to \mbox{knowledge-\textit{why}}~$\psi$.


\begin{example}
    Consider a road passing through a field with a sprinkler in it. The sprinkler is switched on by a weather sensor if it is sunny. Suppose that it rains whenever it is cloudy and that the road is wet if either it rains or the sprinkler is turned on. Finally, suppose that a wet road is slippery. 
    
    Denote by $\textit{cloudy}$ the event that the weather is cloudy, by $\textit{sprinkler}$ the event that the sprinkler is on, by $\textit{rain}$ the event of rainy weather, by $\textit{wet}$ the event that the road is wet, and by $\textit{slippery}$ the event that the road is slippery. 
    
    The described causal knowledge leads to the following causal rules:  
    \begin{align}
        & \textit{cloudy} \Rightarrow \textit{rain} && 
        \label{equation - first rule sprinkler} 
        \neg \textit{cloudy} \Rightarrow \textit{sprinkler} && 
        && \textit{rain} \Rightarrow \textit{wet} && 
        \\&
        \textit{sprinkler} \Rightarrow \textit{wet} 
        && \textit{wet} \Rightarrow \textit{slippery} 
        \label{equation - second rule sprinkler}
    \end{align} 
    \label{example - introduction running example}
    \label{example - causal reasoning is captured in rules}
\end{example} 

Principle~\ref{principle - causal foundation} gives rise to a set of external premises~$\mathcal{E}$, consisting of statements~$\epsilon$ that, if observed, do not require further explanation or \emph{demonstration}. \mbox{\cite{Bochman}} introduces such external premises through \emph{default rules} of the form~$\phi \Rightarrow \phi$, which express that a statement~$\phi$ is self-explanatory. In doing so, he obtains a language for representing Boolean causal models of \cite{Causality}, where causal relationships are modeled by \emph{structural equations}.

\begin{example}
    Example~\ref{example - introduction running example} gives rise to the following default rules:
    \begin{align}
        & \textit{cloudy} \Rightarrow \textit{cloudy},~ \neg \textit{cloudy} \Rightarrow \neg \textit{cloudy}, 
          && \text{(it is either cloudy or not)} 
        \\
        & \neg \textit{sprinkler} \Rightarrow \neg \textit{sprinkler}, && \text{(the sprinkler initially is off)} \\
        & \neg \textit{rain} \Rightarrow \neg \textit{rain} ,~~~~~ 
        \neg \textit{wet} \Rightarrow \neg \textit{wet}, && 
        \neg \textit{slippery} \Rightarrow \neg \textit{slippery} 
        \label{equation - last rule sprinkler}
    \end{align}
    
    \cite{Bochman} models Example~\ref{example - introduction running example} using the causal theory~$\Delta$, consisting of Rules~(\ref{equation - first rule sprinkler})--(\ref{equation - last rule sprinkler}). It corresponds to the causal model $\mathcal{M}$ with structural equations:
    \begin{align}
        & \textit{rain} := \textit{cloudy}, 
        && \textit{sprinkler} := \neg \textit{cl...}, 
        && \textit{wet} := \textit{rain} \lor \textit{sp...}, 
        && \textit{slippery} := \textit{wet}.
        \label{equation - structural equations sprinkler}
    \end{align}
    Intervening and switching the sprinkler off yields  
    the modified model $\mathcal{M}_{\neg \textit{sprinkler}}$:
    \begin{align}
        & \textit{rain} := \textit{cloudy},  
        && \textit{sprinkler} := \textit{False},  
        && \textit{wet} := \textit{rain} \lor \textit{sp...},
        &&  \textit{slippery} := \textit{wet}.
        \label{equation - structural equations sprinkler after interventions}
    \end{align}
    It corresponds to the modified causal theory $\Delta_{\neg \emph{sprinkler}}$ that results from $\Delta$ by replacing the rule $\neg \emph{couldy} \Rightarrow \emph{sprinkler}$ with $\top \Rightarrow \neg \emph{sprinkler}$.
    \label{example - indemonstrable knowledge}
\end{example}


\cite{Bochman} uses propositional logic to reason about \emph{syllogisms}, interpreting the provability operator~${(\vdash)/2}$ as the acquisition of knowledge-\textit{that}. For a set of statements~$\Phi$ and a statement~$\psi$, the expression~${\Phi \vdash \psi}$ reads: ``\mbox{knowledge-\textit{that}}~$\Phi$ leads to knowledge-\textit{that}~$\psi$.''

To capture the acquisition of knowledge-\textit{why}, he extends causal theories~$\Delta$ to an explainability relation~${(\Rrightarrow_{\Delta})/2}$ defined by axioms grounded in \textit{consistency with deduction} in Principle~\ref{principle - consistency with logical reasoning}. \cite{Bochman} then interprets the expression~$\Phi \Rrightarrow_{\Delta} \psi$ as ``knowledge-\emph{that} $\Phi$ explains knowledge-\emph{why} $\psi$.''  

\begin{example}
    In Examples \ref{example - introduction running example} and \ref{example - indemonstrable knowledge}, suppose it is observed that the weather is cloudy and rainy, i.e., there is knowledge-\textit{that} it is cloudy and rainy. 
    
    Since ${\textit{cloudy} \Rightarrow \textit{rain}} \in \Delta$, there is a demonstration of rainy weather based on the premise that it is cloudy. As the default rule ${\textit{cloudy} \Rightarrow \textit{cloudy}} \in \Delta$, $\textit{cloudy}$ is an external premise. Hence, demonstrating~$\textit{cloudy}$ lies beyond the scope of the given area of science—one might argue that it belongs to meteorology. 
    
    Thus, observing cloudy weather and the demonstration of~$\textit{rain}$ from~$\textit{cloudy}$ provides \mbox{knowledge-\textit{why}}~$\textit{rain}$. It follows that~$\textit{cloudy} \Rrightarrow_{\Delta} \textit{rain}$.
\end{example}

To obtain a semantics for causal theories, \cite{Bochman} invokes the principle of \textit{natural necessity}, which Aquinas formulated as follows:

\begin{principle}[Natural Necessity]
    ``... given the existence of the cause, the effect must necessarily follow.'' (Thomas Aquinas: Summa Contra Gentiles II: 35.4; translation by \cite{Anderson})
    \label{principle - Aquinas}
\end{principle}   

\begin{example}
    In Example~\ref{example - introduction running example}, this means, for instance, that the road is wet whenever it rains.
    \label{example - Mill's covering law}
\end{example}

Furthermore, \cite{Bochman} asserts  the assumption of \textit{sufficient causation}, which Leibniz formulated as follows:

\begin{assumption}[Sufficient Causation]
    ``...there is nothing without a reason, or no effect without a cause.''  
    (Gottfried Wilhelm Leibniz: First Truths; translation by 
\cite{Leibniz}, p.~268)
    \label{principle - Leibniz}
\end{assumption}

\begin{example}
    In Example~\ref{example - introduction running example}, this implies that rain does not occur without a cause. Therefore, if rain is observed, it must be cloudy. Assumption~\ref{principle - Leibniz} ensures that all possible occurrences are explained by the given area of science.

    Principle \ref{principle - Aquinas} and Assumption~\ref{principle - Leibniz} imply that the causal theory~$\Delta$ in Example~\ref{example - indemonstrable knowledge} corresponds to the states described by the sets:
    $\omega_1 :=  \{ \textit{sprinkler},~ \textit{wet},~ \textit{slippery} \}$ and $\omega_2 := 
    \{ \textit{cloudy},~ \textit{rain},~ \textit{wet},~ \textit{slippery} \}$, i.e., the solutions of Equations (\ref{equation - structural equations sprinkler}).
    \label{example - causal worlds}
    \label{example - causal law}
\end{example}

\subsubsection{First Contribution: From Bochman's Theory to Cyclic Causal Relations}

The approach of \cite{Causality} and \mbox{\cite{Bochman}} leads to counterintuitive results in the presence of cyclic causal relationships.

\begin{example}
    In Example~\ref{example - houses introduction},  denote by $\textit{start\_fire}(h_i)$, $i = 1, 2$ that House~$h_i$ starts to burn, and by $\textit{fire}(h_i)$ that House~$h_i$ is burning.  Accepting $\textit{start\_fire}(h_i)$ as external premises, the situation is captured by the following causal theory~$\Delta$:
    
    \begin{align}
        & \textit{fire}(h_2) \Rightarrow \textit{fire}(h_1), && \textit{fire}(h_1) \Rightarrow \textit{fire}(h_2),
        \label{equation 1 - counterexample Bochman}\\
        & \textit{start\_fire}(h_1) \Rightarrow \textit{fire}(h_1), && \textit{start\_fire}(h_2) \Rightarrow \textit{fire}(h_2),
        \label{equation 2 - counterexample Bochman} \\
        & \textit{start\_fire}(h_1) \Rightarrow \textit{start\_fire}(h_1), && \textit{start\_fire}(h_2) \Rightarrow \textit{start\_fire}(h_2), \\
        & \neg \textit{fire}(h_1) \Rightarrow \neg \textit{fire}(h_1), && \neg \textit{fire}(h_2) \Rightarrow \neg \textit{fire}(h_2), \\
        & \neg \textit{start\_fire}(h_1) \Rightarrow \neg \textit{start\_fire}(h_1), && \neg \textit{start\_fire}(h_2) \Rightarrow \neg \textit{start\_fire}(h_2).
        \label{equation - last rule counterexample Bochman}
    \end{align}


    Rules~(\ref{equation 1 - counterexample Bochman}) imply both ${\textit{fire}(h_1) \Rrightarrow_{\Delta} \textit{fire}(h_1)}$ and ${\textit{fire}(h_2) \Rrightarrow_{\Delta} \textit{fire}(h_2)}$, contradicting Principles \ref{principle - non-circularity} and~\ref{principle - causal foundation}. 
    Furthermore, the causal theory~$\Delta$ and the corresponding causal model admit a possible world in which both houses burn, even though neither of them initially caught fire. This contradicts everyday intuition: houses do not catch fire merely because they can potentially ignite one another.
    \label{example - counterexample Bochman introduction}
\end{example}

To address this issue, areas of science are represented as \emph{causal systems}: 
\[
\textbf{CS} := (\Delta, \mathcal{E}, \mathcal{O}) \, .
\]

Here,~$\Delta$ is a causal theory that formalizes the causal knowledge within a given area of science, as illustrated in Example~\ref{example - introduction running example};~$\mathcal{E}$ is a set of statements interpreted as \emph{external premises}; and~$\mathcal{O}$ is a set of observations, i.e., \mbox{knowledge-\emph{that}}, which the system~$\textbf{CS}$ uses to reason within this area of science. 


\begin{remark}
    Causal theories $\Delta$ may also mention expressions like~${\top \Rightarrow \phi}$ for a statement $\phi$. The best interpretation of this construct is that the truth of $\phi$ is directly obtained from the essences in Aristotle's ontology. Thus, this statement expresses that \textit{nous} into the essence underlying $\phi$ is needed. 
    \label{remark - metaphysical justification}
\end{remark}

\begin{remark}
     A causal theory $\Delta$ may also mention expressions like~${\phi \Rightarrow \bot}$ for a statement $\phi$. Such expressions are excluded within an Aristotelian science, in which all the premises of demonstrations are true, as discussed in Section~\ref{subsubsec: Knowledge that and knowledge why}.
     \label{remark - observations}
\end{remark}

The system then determines possible states of the world according to Principles~\ref{principle - causal foundation} and~\ref{principle - Aquinas}, as well as Assumption~\ref{principle - Leibniz}. If it does so by additionally adhering to \emph{directionality} in Principle~\ref{principle - directionality}, the system is said to acquire knowledge-\emph{why}.

\begin{example}
    Example~\ref{example - houses introduction} is represented by the system~$\textbf{CS} := (\Delta, \mathcal{E}, \mathcal{O})$, where~$\Delta$ is the causal theory consisting of Rules~(\ref{equation 1 - counterexample Bochman}) and~(\ref{equation 2 - counterexample Bochman}) from Example~\ref{example - counterexample Bochman introduction}; the set of observations is~$\mathcal{O} = \emptyset$; and the external premises are given by
    \[
        \mathcal{E} := \left\{ 
            \textit{start\_fire}(X),~ \neg \textit{start\_fire}(X),~ \neg \textit{fire}(X) \mid X \in \{ h_1, h_2 \} 
        \right\} \, .
    \]

    Committing to Principle~\ref{principle - causal foundation}, the system~$\textbf{CS}$ does not consider the explanations~${\textit{fire}(h_1) \Rrightarrow_{\Delta} \textit{fire}(h_1)}$ and~${\textit{fire}(h_2) \Rrightarrow_{\Delta} \textit{fire}(h_1)}$ to be causal. Hence, \emph{sufficient causation} in Assumption~\ref{principle - Leibniz} rules out the world in which both houses are burning while neither has initially caught fire. 
    
    As the system $\textbf{CS}$ has no observations, it possesses knowledge-\emph{why}.
    \label{example - introducing causal systems}
\end{example}

\subsubsection{Second Contribution: On Feasibility of Interventions}

Let $\textbf{i}$ be an intervention such that the corresponding modified causal system~${\textbf{CS}_{\textbf{i}} := (\Delta_{\textbf{i}}, \mathcal{E}, \mathcal{O})}$ possesses knowledge-\emph{why}. This work then argues that the following principle of \emph{non-interference} is satisfied, which justifies the use of~$\textbf{CS}_{\textbf{i}}$ for predicting  the effect of Intervention $\textbf{i}$. 

\begin{principle}[Non-Interference]
    Effects of an intervention~$\textbf{i}$ propagate only along the causal direction and do not influence unrelated external premises.
    \label{principle - non-interference}
    \label{assumption - feasibility of external interventions}
\end{principle}

\begin{example}
    The causal system $\textbf{CS}$ in Example \ref{example - introducing causal systems} possesses knowledge about external interventions.
\end{example}

\subsubsection{Third Contribution: A Theory of Causal Reasoning Under Uncertainty}

Note that \emph{natural necessity} in Principle~\ref{principle - Aquinas} may fail in real-world situations:

 \begin{example}
     In Example \ref{example - causal reasoning is captured in rules}, suppose that it is cloudy and raining. The clouds are clearly a cause for the rain. However, it is not necessarily the case that clouds lead to rain, as one can easily imagine a heavily overcast day without~rain. 
 \end{example}

To weaken the notion of \textit{natural necessity} in Principle~\ref{principle - Aquinas}, this work introduces uncertainty about whether it applies to the rules in a causal theory~$\Delta$. Saying that \textit{natural necessity} applies to ${\phi \Rightarrow \psi}$ with probability~$\pi$ means that the material implication~${\phi \rightarrow \psi}$ holds with probability~$\pi$.
 
Assume that each statement~$\phi_i$,~$1 \leq i \leq n$ holds with probability~$\pi_i$. Instead of deducing knowledge-\textit{that} via provability~$(\vdash)/2$, a probabilistic analogue is obtained by applying the principle of \emph{maximum entropy} of~\cite{Shannon}.

Maximizing entropy~$H(\pi)$ under the constraint that each~$\phi_i$ holds with probability~$\pi_i$ for $1 \leq i \leq n$ generally does not yield a distribution~$\pi$ that can be easily expressed in terms of the~$\pi_i$. This work therefore adopts Parametrization~\ref{parametrization - maximum entropy model}, used in the LogLinear models of~\cite{MLN}, and represents uncertainty about the~$\phi_i$ through weights~${w_i \in \mathbb{R} \cup \{ \pm \infty \}}$.

These considerations lead to the notion of a \emph{maximum entropy causal system:}
\[
\textbf{CS} := (\Theta, \mathcal{E}, \mathcal{O}, \Sigma)
\]
Here,~$\Theta$ is a set of weighted causal rules~$(w, \phi \Rightarrow \psi)$, where~$w \in \mathbb{R} \cup \{ \pm \infty \}$ expresses uncertainty about the \emph{natural necessity} of~$\phi \Rightarrow \psi$ within an area of science;~$\mathcal{E}$ is a set of \emph{external premises};~$\mathcal{O}$ a set of observations; and~$\Sigma$ a set of weighted constraints~$(w, \phi)$, which represent results from superordinate areas of science, where~$\phi$ is a statement and~$w \in \mathbb{R} \cup \{ \pm \infty \}$ is its associated weight.

\begin{remark}
    The choice of the letters $\Theta$, $\mathcal{E}$, $\mathcal{O}$, and $\Sigma$ is mnemonic: they spell ``Theos,'' the Greek word for ``God.''
\end{remark}

In Bayesian networks, \cite{Williamson2001} notes that maximizing entropy conflicts with the following  consequence of \emph{directionality} in Principle~\ref{principle - directionality}:

\begin{principle}[Causal Irrelevance]
    Unobserved non-causes do not change~\mbox{beliefs}.
    \label{principle - causal irrelevance}
\end{principle}

This work follows \cite{Williamson2001} in resolving this conflict in Formalization \ref{formalization - causal irrelevance} and argues that an analogue to knowledge-\emph{why} is obtained if the entropy is maximized by additionally accounting for Principles \ref{principle - causal foundation}, \ref{principle - causal irrelevance} and Assumption~\ref{principle - Leibniz}. 

This leads to the \emph{causal semantics} of \emph{maximum entropy causal systems} in Definition \ref{definition - causal semantics} and Formalization \ref{formalization - knowledge-why under uncertainty}, which provide probabilistic counterparts to Aristotle’s notions of \emph{demonstration} and knowledge-\textit{why}.

To demonstrate the approach's effectiveness, this work analyzes the knowledge captured by the Bayesian networks and causal models of \cite{Causality}. In particular, \cite{Williamson2001} shows that a Bayesian network's distribution arises by extending its probabilistic information and greedily maximizing entropy along the causal order. Hence, Bayesian networks provide a probabilistic analogue of Aristotle's knowledge-\textit{why}, enabling them to answer queries about external interventions. 

Thus, implicit philosophical assumptions in \cite{Causality} are made explicit.

\begin{example}
Modify Example~\ref{example - introduction running example} by assuming the sprinkler is activated by a weather sensor with probability~$0.1$ if it is cloudy and~$0.7$ otherwise. It rains with probability~$0.6$ when cloudy. If it rains or the sprinkler is on, the pavement gets wet with probability~$0.9$. If the pavement is wet, the road is slippery with probability~$0.8$.

\cite{Causality} represents this mechanism by the causal model $\mathcal{M}$:
\begin{align*}
& \textit{cloudy} := u_1 
&& \textit{rain} := \textit{cl...} \land u_2  &&  \textit{sprinkler} := ( \textit{cl...}  \land u_3) \lor (\neg \textit{cl...} \land u_4)  \\
&&&\textit{wet}   := ( \textit{rain}  \lor \textit{spr...}) \land u_5    &&  \textit{slippery}   :=  \textit{wet}   \land u_6   
\end{align*}

To represent the uncertainties in the story, he specifies the probabilities: ${\pi(u_1) = 0.5}$, ${\pi(u_2) = 0.6}$, ${\pi(u_3) = 0.1}$, ${\pi(u_4) = 0.7}$, ${\pi(u_5) = 0.9}$ and ${\pi(u_6) = 0.8}$. 
Asserting that $u_1,...,u_6$ are mutually independent random variables defines a unique distribution $\pi$, resulting in the probabilistic causal model $\mathbb{M} := (\mathcal{M}, \pi)$. 

Here, uncertainty arises from hidden variables, implicitly represented in the error terms~$\textbf{U} := \{ u_1,\dots,u_6 \}$. For example,~$u_3$ summarizes potential causes—such as sensor failure—for the sprinkler being on despite cloudy weather. These factors are not modeled explicitly but are captured by~$\textbf{U}$ and~$\pi$.

The model $\mathbb{M}$ yields the maximum entropy causal system without observations ${\textbf{CS}(\mathbb{M}) := (\Theta, \mathcal{E}, \emptyset, \Sigma)}$, where $\Theta$ consists of rules with infinite~weight:
\begin{align}
& (+ \infty,  u_1 \Rightarrow \textit{cloudy})   
&& (+ \infty, \textit{cloudy} \land u_2 \Rightarrow  \textit{rain} )   \\
&  ( + \infty, \textit{cloudy}  \land u_3  \Rightarrow  \textit{sprinkler} )     
&&  ( + \infty, \neg \textit{cloudy} \land u_4  \Rightarrow  \textit{sprinkler} ) \\ &
( + \infty , \emph{rain} \Rightarrow \textit{wet}) &&  ( + \infty , \emph{sprinkler} \Rightarrow \textit{wet})   \\
&  (+ \infty,  \emph{wet} \land u_6 \Rightarrow \textit{slippery});  
\end{align}
the external premises are $\mathcal{E} := \{ u_1, \neg u_1,...,u_6, \neg u_6, \neg \textit{cloudy},..., \neg \textit{sprinkler}  \}$; and $\Sigma := \{ (\ln(0.5 \cdot 0.6 \cdot ... \cdot 0.8), u_1 \land ... \land u_6 ) ,...,
(\ln(0.5 \cdot 0.4 \cdot ... \cdot 0.2), \neg u_1 \land ... \land \neg u_6 )\}$. It possesses knowledge-\emph{why} and can reason on interventions as desired.

To avoid introducing hidden variables, \cite{Causality} models the situation with a Bayesian network~$\textbf{BN} := (G, \pi(\cdot \mid \pa(\cdot)))$, where~$G$ is a directed acyclic graph and~$\pi(\cdot \mid \pa(\cdot))$ the associated parameters:

\begin{equation}
\begin{tikzcd}
        & \textit{rain} \arrow[rd] &               &          \\
        \textit{cloudy} \arrow[ru] \arrow[r] & \textit{sprinkler} \arrow[r] & \textit{wet} \arrow[r] & \textit{slippery}        
    \end{tikzcd}
    \label{equation - causal structure}
\end{equation}
\begin{align}
        \nonumber \pi (\textit{cloudy}) &= 0.5, \\
        \pi (\textit{rain} \vert \textit{cloudy}) &= 0.6, \quad &\nonumber \pi (\textit{rain} \vert \neg \textit{cloudy}) &= 0, \\
        \nonumber \pi (\textit{sprinkler} \vert \textit{cloudy}) &= 0.1, \quad &\pi (\textit{sprinkler} \vert \neg \textit{cloudy}) &= 0.7, \\ \nonumber
        \pi (\textit{wet} \vert \textit{rain}, \textit{sprinkler}) &= 0.9, \quad &\pi (\textit{wet} \vert \neg \textit{rain}, \textit{sprinkler}) &= 0.9, \\ \nonumber
        \pi (\textit{wet} \vert \textit{rain}, \neg \textit{sprinkler}) &= 0.9, \quad &\pi (\textit{wet} \vert \neg \textit{rain}, \neg \textit{sprinkler}) &= 0, \\
        \pi (\textit{slippery} \vert \textit{wet}) &= 0.8, \quad &\pi (\textit{slippery} \vert \neg \textit{wet}) &= 0.
        \label{equation - parameters of a BN}
\end{align}
The Bayesian network $\textbf{BN}$ corresponds to the maximum entropy causal system without observations $\textbf{CS}(\textbf{BN}) := (\Theta, \mathcal{E}, \emptyset, \Sigma)$, where $\Sigma := \{ (\ln(0.5), \emph{cloudy}) \}$; $\mathcal{E} := \{ \emph{cloudy}, \neg \emph{cloudy} , \neg \emph{rain}, ..., \neg \emph{slippery}\}$;  and $\Theta$ is given by   
\begin{align*}
        &(w_1,~ \textit{cloudy} \Rightarrow \textit{rain})  && (- \infty,~ \neg \textit{cloudy} \Rightarrow \textit{rain} ), \\
        &(w_2,~ \textit{cloudy} \Rightarrow \textit{sprinkler})  && (w_3,~ \neg \textit{cloudy} \Rightarrow \textit{sprinkler}) \\
        &(w_4,~ \textit{rain}\land \textit{sprinkler} \Rightarrow \textit{wet}  ) && (w_5,~ \neg \textit{rain} \land \textit{sprinkler} \Rightarrow \textit{wet} )  \\
        &(w_6,~ \textit{rain}\land \neg \textit{sprinkler} \Rightarrow \textit{wet} ) &&  (- \infty,~ \neg \textit{rain} \land \neg \textit{sprinkler} \Rightarrow \textit{wet})  \\
        &(w_7,~ \textit{wet} \Rightarrow \textit{slippery} ) &&
        (- \infty,~ \neg \textit{wet} \Rightarrow \textit{slippery}  )
\end{align*}
for properly chosen weights $w_1,...,w_7 \in \mathbb{R}$. 
\label{example - introduction - sprinkler}
\end{example}

\section{Knowledge-\emph{Why} in a Deterministic Boolean Setting}
\label{sec: Knowledge in Deterministic Systems}

This section investigates causal reasoning in deterministic Boolean settings.

\subsection{Preliminaries}
\label{subsec: Preliminaries}

Here, the necessary prerequisites are gathered by recalling the fundamentals of propositional logic and the theory of causality as presented by~\mbox{\cite{Bochman}}.

\subsubsection{Propositional Logic: A Language for Knowledge-\emph{That}}
\label{subsubsec: Propositional Logic}

Propositional logic provides a framework for reasoning about truth, that is, for specifying sets of Boolean functions that satisfy certain constraints. This work adopts the standard notations of propositions, (propositional) formulas, and structures, which are identified with the sets of propositions that are true in them, as laid out, for instance, by~\mbox{\citet{sep-logic-propositional}}.



\begin{example}
    To formalize reasoning in Example~\ref{example - introduction running example}, the propositional alphabet  
    $
    \mathfrak{P} := \{ \textit{cloudy},~ \textit{rain},~ \textit{sprinkler},~
    \textit{wet},~\textit{slippery} \}
    $
    is introduced. 
    \label{example - propositional alphabet}

    The structure $\omega_2$ in Example \ref{example - causal worlds} represents the complete state description: 
    $\textit{cloudy} \mapsto True,~ \textit{sprinkler} \mapsto False$, $\textit{slippery} \mapsto True$,
        $\textit{rain} \mapsto True$, ${\textit{wet} \mapsto Tr...
    }$
    \label{example - propositional structure}
\end{example}

Propositional formulas are connected by the semantic notion of entailment, denoted $(\models)/2$, and the syntactic notion of derivation, denoted $(\vdash)/2$ . 

 \begin{definition}[Semantic Entailment, World]
    A set of formulas~$\Phi$ is called \textbf{deductively closed} if $\psi \in \Phi$ whenever $\Phi \models \psi$. The \textbf{deductive closure} of~$\Phi$ is the smallest deductively closed set~$\bar{\Phi}$ such that $\Phi \subseteq \bar{\Phi}$. 
    
    A \textbf{world} is a consistent, deductively closed set of formulas that is maximal with respect to set inclusion.
    \label{definition - entailment}
\end{definition}

\begin{remark}
    Every world~$\Phi = \overline{\textbf{L}}$ is the deductive closure of the set of its literals~$\textbf{L}$. 
    Since~$\Phi$ is consistent and maximal with respect to set inclusion, the set~$\textbf{L}$ can be expressed as
    $
        \textbf{L} := (\textbf{L} \cap \mathfrak{P}) \cup \{ \neg p \mid p \in \mathfrak{P},~ p \notin \Phi \}.
    $
    Consequently, this work identifies~$\Phi$ with the set of propositions~$\textbf{L} \cap \mathfrak{P}$, which also serves as a synonym for structures.
\end{remark}

Propositional logic has a sound and complete deductive calculus, first given by \cite{Frege1879}. The \emph{completeness theorem} ensures that semantic entailment matches syntactic derivability.

\begin{theorem}[Completeness Theorem]
    A formula is semantically entailed by a set of formulas if and only if it is derivable from that set.  
\end{theorem}


\subsubsection{Bochman's Logical Theory of Causality}
\label{subsubsec: Bochman's Logical Theory of Causality}

\citet{Bochman} proposes a formalization of knowledge-\textit{why}, as introduced in Section~\ref{subsubsec: Knowledge that and knowledge why}. He considers a system or agent that uses causal knowledge to explain instances of knowledge-\textit{that} about the world. \textit{Explainability} is formalized as a binary relation~$(\Rrightarrow)/2$, which captures how one instance of knowledge-\textit{that} can be explained in terms of others. Ultimately, \citet{Bochman} characterizes knowledge-\textit{why} as the subset of knowledge-\textit{that} that is justified through this relation of explainability.

To begin, he adopts propositional formulas in an alphabet~$\mathfrak{P}$ and uses the provability operator~$(\vdash)/2$ to formalize reasoning about knowledge-\textit{that}.

\begin{Language}
    Formulas in the alphabet $\mathfrak{P}$ represent knowledge-\emph{that} and the provability operator $(\vdash)/2$ represents the existence of \emph{syllogisms}.
    \label{language - propositional logic}
\end{Language}

\begin{example}
In Example~\ref{example - introduction running example}, suppose one assumes or observes knowledge-\textit{that} it is cloudy whenever it rains, i.e.,~${
\textit{rain} \rightarrow \textit{cloudy}
}$. 
If, in addition, it is observed that it is raining, then
$
\{ \textit{rain},\, \textit{rain} \rightarrow \textit{cloudy} \} \vdash \textit{cloudy},
$
that is, one deduces \textit{that} it is cloudy.
However, $\textit{rain}$ does not constitute an explanation for $\textit{cloudy}$, and thus one does not possess knowledge-\textit{why} it is cloudy.
\label{example - fire formalization}
\end{example}

Example~\ref{example - fire formalization} illustrates that material implication ``$\rightarrow$'' and the provability operator $(\vdash)/2$ generally do not capture causal knowledge and demonstrations. \cite{Bochman} therefore formalizes \textit{explainability} as a binary relation on knowledge-\textit{that}; that is, he makes the following assumption:

\begin{Language}
    Explainability~$(\Rrightarrow)/2$ is a binary relation on \mbox{formulas}. For formulas~$\phi$ and~$\psi$, the expression~$\phi \Rrightarrow \psi$ indicates that \mbox{knowledge-\textit{that}}~$\phi$ leads to knowledge-\textit{why}~$\psi$.
    \label{language - causal reasoning}
\end{Language}

\begin{example}
In Examples~\ref{example - introduction running example} and \ref{example - propositional alphabet}, explainability is formalized as a binary relation~${(\Rrightarrow)/2}$ on formulas over the alphabet~$\mathfrak{P}$. For instance, the statement “rain or sprinkler explains why the road is wet and slippery” is expressed as
${
(\textit{rain} \lor \textit{sprinkler}) \Rrightarrow (\textit{wet} \land \textit{slippery}).
}$
This relation is not a logical connective, and nested expressions like
${
\textit{cloudy} \Rrightarrow (\textit{rain} \Rrightarrow \textit{wet})
}$
lack semantic interpretation.
\label{example - causal reasoning as binary relation}
\end{example}



\cite{Bochman} interprets \emph{consistency with deduction} in Principle~\ref{principle - consistency with logical reasoning} as explainability~$(\Rrightarrow)/2$ being a production inference relation:

\begin{definition}[Production Inference Relation]
A \textbf{production inference relation} is a binary relation $(\Rrightarrow)/2$ on the set of formulas in the alphabet $\mathfrak{P}$ that satisfies the following assertions for all propositional formulas $\phi$, $\psi$ and~$\rho$:
\begin{enumerate}
    \item[i)] If  $\phi \vdash \psi$ and $\psi \Rrightarrow \rho$, then $\phi \Rrightarrow \rho$ follows.  \textbf{(Strengthening)}
    \item[ii)] If  $\phi \Rrightarrow \psi$ and $\psi \vdash \rho$, then $\phi \Rrightarrow \rho$ follows.  \textbf{(Weakening)}
    \item[iii)] If  $\phi \Rrightarrow \psi$ and $\phi \Rrightarrow \rho$, then $\phi \Rrightarrow \psi \land \rho$ follows. \textbf{(And)}
    \item[iv)] One has $\top \Rrightarrow \top$ and $\bot \Rrightarrow \bot$. \textbf{(Truth and Falsity)}
\end{enumerate}
Note that the propositional formulas $\phi$, $\psi$ and $\rho$ do not mention the binary relation $(\Rrightarrow)/2$. If $\phi \Rrightarrow \psi$ for two formulas $\phi$ and $\psi$, the formula~$\phi$ is said to \textbf{explain} $\psi$ or that~$\phi$ is an \textbf{explanans} of~$\psi$ or that~$\psi$ is an \textbf{explanandum} of~$\phi$. 

Given a production inference relation~$(\Rrightarrow)/2$, write~$\Phi \Rrightarrow \psi$ for a set of propositional formulas~$\Phi$ and a formula~$\psi$ if there exists a finite subset~$\Phi' \subseteq \Phi$ such that~$\displaystyle \bigwedge_{\phi \in \Phi'} \phi  \Rrightarrow \psi$ holds.

The \textbf{consequence operator}~$\mathcal{C}$ is then defined by assigning to each set of propositional formulas~$\Phi$ the set
$
\mathcal{C} (\Phi) := \{ \psi \text{ propositional formula} : \Phi \Rrightarrow \psi \}.
$
Note that both~$\Phi$ and~$\mathcal{C}(\Phi)$ are sets of formulas that do not themselves mention the relation~$(\Rrightarrow)/2$.

\label{definition - production inference relation}
\end{definition}

\begin{formalization}
    A binary relation on propositional formulas $(\Rrightarrow)/2$ satisfies \emph{consistency with deduction} in Principle~\ref{principle - consistency with logical reasoning} if and only if it is a production inference relation.
    \label{formalization - consistency with deduction}
\end{formalization}

Following Language \ref{language - causal reasoning}, the causal operator~$\mathcal{C}(\Phi)$ denotes the \mbox{knowledge-\textit{why}} resulting from knowledge-\textit{that} the propositions in~$\Phi$ are true.



In the notion of a binary semantics, \cite{Bochman} formalizes the distinction between knowledge-\textit{that} and \mbox{knowledge-\textit{why}} in Section~\ref{subsubsec: Knowledge that and knowledge why}.

\begin{definition}[Bimodel, Binary Semantics]
    A pair $(\Phi, \Psi)$ of consistent deductively closed sets of formulas~$\Phi$ and $\Psi$ is called a \textbf{(classical) bimodel}. A \textbf{(classical) binary semantics} $\mathcal{B}$ then is a set of bimodels. 
    
    The expression $\psi \Rrightarrow \phi$ is \textbf{valid} in a bimodel $(\Phi, \Psi)$ if either~${\phi \not \in \Phi}$ or $\psi \in \Psi$, i.e., $\phi \in \Phi$ only if $\psi \in \Psi$. Finally, the expression $\psi \Rrightarrow \phi$ is \textbf{valid} in a binary semantics~$\mathcal{B}$ if it is valid in all bimodels~${(\Phi, \Psi) \in \mathcal{B}}$. 
    \label{definition - binary semantics}
\end{definition}   



Fortunately, production inference relations and binary semantics correspond to each other.

\begin{theorem}
    To each binary semantics~$\mathcal{B}$, one can associate a production inference relation~$(\Rrightarrow_{\mathcal{B}})/2$ defined by the condition that~$\psi \Rrightarrow_{\mathcal{B}} \phi$ holds whenever~$\psi \Rrightarrow \phi$ is valid in~$\mathcal{B}$. 
    
    Conversely, every production inference relation~$(\Rrightarrow)/2$ induces a \textbf{canonical binary semantics}
    \[
    \mathcal{B}_{\Rrightarrow} :=
    \{ (\mathcal{C}(\Phi), \Phi) : \Phi \text{ a consistent and deductively closed set of formulas} \}.
    \]
    A binary relation~$(\Rrightarrow)/2$ qualifies as a production inference relation if and only if it is determined by its canonical binary semantics.
    \label{theorem - production inference relations and binary semantics}
\end{theorem}

\begin{proof}
    \cite{bochman2005} proves this result in Lemma 8.3 and Theorem 8.4.
\end{proof}

Next, \cite{Bochman} asserts that explainability $(\Rrightarrow)/2$ satisfies \textit{natural necessity} in Principle \ref{principle - Aquinas}.
This means that once a formula~$\phi$ is explained and one knows \textit{why} $\phi$ holds, one also knows \textit{that} $\phi$ holds. Hence, explainability gives rise to a consistent binary semantics: 

\begin{definition}[Consistent Binary Semantics]
    A bimodel $(\Phi, \Psi)$ is \textbf{consistent} if $\Phi \subseteq \Psi$. A binary semantics $\mathcal{B}$ is \textbf{consistent} if all bimodels ${(\Phi, \Psi) \in \mathcal{B}}$~are. 
    \label{definition - consistent binary semantics}
\end{definition}

Consistency of binary semantics corresponds to the following property of production inference relations:

\begin{definition}[Regular Production Inference Relation]
A \textbf{regular} production inference relation~$(\Rrightarrow)/2$ is a production inference relation that satisfies the following property for all formulas $\phi$, $\psi$ and $\rho$:
\begin{enumerate}
     \item[] If $\phi \Rrightarrow \psi$ and $\phi \land \psi \Rrightarrow \rho$ holds, then $\phi \Rrightarrow \rho$ is valid. \textbf{(Cut)}
\end{enumerate}
\label{definition - regular production infernece relation}
\end{definition}



\begin{theorem}[\cite{bochman2005}, Theorem 8.9]
A production inference relation~$(\Rrightarrow)/2$ is regular if and only if it is generated by a consistent binary semantics. In particular, the canonical binary semantics $\mathcal{B}_{\Rrightarrow}$ is consistent.~$\square$  
\label{theorem - regular production inference relations}
\end{theorem}

\cite{Bochman} further commits to \textit{sufficient causation}, as stated in Assumption~\ref{principle - Leibniz}. Together with \textit{natural necessity} in Principle~\ref{principle - Aquinas}, this implies that within an area of science knowledge-\textit{that} coincides with \mbox{knowledge-\textit{why}}.

If explainability is a production inference relation~${(\Rrightarrow)/2}$, as enforced in Language \ref{language - causal reasoning} and Formalization \ref{formalization - consistency with deduction}, Assumption~\ref{principle - Leibniz} and Theorem~\ref{theorem - regular production inference relations} establish that all possible states of knowledge-\textit{why} correspond to exact theories.

\begin{definition}[Exact Theory]
    An \textbf{exact theory} of a production inference relation~${(\Rrightarrow)/2}$ is a deductively closed set $\Phi$ such that~${\mathcal{C} (\Phi) = \Phi}$, that is, 
    ${
    \mathcal{C} (\Phi) \subseteq \Phi~(\textit{natural necessity})~\text{and}~
    \mathcal{C} (\Phi) \supseteq \Phi~(\textit{sufficient causation})
    }$.
    \label{definition - exact theory of a production inference relation}
\end{definition}

\textit{Sufficient causation} asserts that all knowledge-\textit{that} is explainable. For our purposes, this includes knowledge-\textit{that}~about ``$\perp$''. 

Recall that a world is a consistent, deductively closed set maximal under inclusion. If a world~$\omega$ is not an exact theory of~$(\Rrightarrow)/2$, then by \textit{sufficient causation}, it cannot occur. Applying the principle again yields~${\omega \Rrightarrow \perp}$, making~$\mathcal{C}(\omega)$ inconsistent. Thus, knowledge-\textit{why} is represented by causal worlds.

\begin{definition}[Causal Worlds Semantics]
    A \textbf{causal world} of a production inference relation $(\Rrightarrow)/2$ is a world~$\omega$ that is an exact theory. 
    The set of all causal worlds $\causal(\Rrightarrow)$ is called the \textbf{causal worlds semantics} of~$(\Rrightarrow)/2$. 
    \label{definition - causal worlds}
\end{definition}


Suppose that $\phi \Rrightarrow \rho$ and $\psi \Rrightarrow \rho$ for propositional formulas~$\phi$,~$\psi$, and~$\rho$. Additionally, assume knowledge-\textit{that}~${\phi \lor \psi}$. Hence, it is known \textit{that} either $\phi$ or~$\psi$ holds, and in both cases, one deduces knowledge-\textit{that}~$\rho$ using \textit{natural necessity} in Principle~\ref{principle - Aquinas}. Invoking \textit{sufficient causation} in Assumption~\ref{principle - Leibniz}, one concludes to knowledge-\textit{why} $\rho$ and obtains~${\phi \lor \psi \Rrightarrow \rho}$.

    

Hence, by \textit{natural necessity} in Principle \ref{principle - Aquinas} and \textit{sufficient causation} in Assumption \ref{principle - Leibniz}, \textit{explainability} is represented by basic production inference relations. 

\begin{definition}[Basic Production Inference Relation]
    A \textbf{basic} production inference relation~${(\Rrightarrow)/2}$ is one that satisfies the following property:
    \begin{enumerate}
         \item[] If $\phi \Rrightarrow \rho$ and $\psi \Rrightarrow \rho$ holds, then $ \phi \lor \psi  \Rrightarrow \rho$ is valid. \textbf{(Or)}
     \end{enumerate} 
     Note that this is equivalent to asserting that $\mathcal{C}(\Phi \cap \Psi) =\mathcal{C}(\Phi) \cap \mathcal{C}(\Psi)$ for all sets of propositional formulas $\Phi$ and $\Psi$.
    \label{definition - basic prodoction inference relation}
\end{definition}

The binary relations on propositional formulas that can represent \textit{explainability} in the sense of~\cite{Bochman} have now been characterized. This gives rise to the definition of causal production inference relations:

\begin{definition}[Causal Production Inference Relation]
    A \textbf{causal} production inference relation is one that is both regular and basic.
     \label{definition - causal production inference relation}
\end{definition}

To summarize, \textit{natural necessity} in Principle~\ref{principle - Aquinas} and \textit{sufficient causation} in Assumption~\ref{principle - Leibniz} imply that knowledge-\textit{that} and knowledge-\textit{why} coincide within an area of science. By formalizing \textit{explainability} through production inference relations $(\Rrightarrow)/2$, as enforced in Languages \ref{language - propositional logic}, \ref{language - causal reasoning}, and Formalization \ref{formalization - consistency with deduction}, one finds that the possible states of knowledge-\textit{why} correspond to the causal worlds in Definition~\ref{definition - causal worlds}. According to Definition~\ref{definition - basic prodoction inference relation}, we conclude that \cite{Bochman} applies the following assertion:

\begin{formalization}
    The production inference relation $(\Rrightarrow)/2$ satisfies \emph{natural necessity} in Principle~\ref{principle - Aquinas} and \emph{sufficient causation} in Assumption~\ref{principle - Leibniz} if and only if it corresponds to the binary semantics 
    $$
    \mathcal{B} := \{ (\Phi, \Phi) \text{: } \Phi \subseteq \omega \text{ for a } \omega \in \Causal (\Rrightarrow) \text{ and }~ \Phi = \bigcap_{\substack{\Phi \subseteq \omega \\ \omega \in \Causal (\Rrightarrow)}} \mathcal{C}(\omega)  
    \}.
    $$
    Hence, causal reasoning $(\Rrightarrow)/2$ is uniquely determined by its causal worlds. Theorem~\ref{theorem - regular production inference relations} and Definition~\ref{definition - basic prodoction inference relation} yield that $(\Rrightarrow)/2$ is causal.
   \label{formalization - Leibniz}
   \label{formalization - causal sufficiency}
\end{formalization}

Recall from Section \ref{subsubsec: Knowledge that and knowledge why}, that an area of science describing a given situation gives rise to a set of external propositions $\mathcal{E}$ that do not require further explanation. Language~\ref{language - propositional logic} then yields that $\mathcal{E}$ is a set of propositional formulas. \textit{Causal foundation} in Principle~\ref{principle - causal foundation} states that a causal explanation should start from these external premises. Hence, a possible world $\omega$ should be explained by the external premises in $\omega \cap \mathcal{E}$. We apply the following assertion:

\begin{formalization}
    Assume the production inference relation $(\Rrightarrow)/2$ satisfies \emph{natural necessity} in Principle~\ref{principle - Aquinas} and \emph{sufficient causation} in Assumption~\ref{principle - Leibniz} according to Formalization~\ref{formalization - Leibniz}. 
    The production inference relation $(\Rrightarrow)/2$ then satisfies \emph{causal foundation} in Principle~\ref{principle - causal foundation} if and only if every causal world~${\omega \in \Causal(\Rrightarrow)}$ is explained by the external premises in $\omega \cap \mathcal{E}$, i.e., $\omega \cap \mathcal{E} \Rrightarrow \omega$.
    \label{formalization - binary semantics}
\end{formalization}


So far, as illustrated in Example~\ref{example - causal reasoning as binary relation}, \textit{explainability} has been represented by specifying the entire binary relation between explanans and explanandum within a causal production inference relation. \cite{Bochman} further applies \emph{causal rules} in Principle~\ref{principle - causal rules}. He concludes that causal production inference relations should be stated as causal rules and theories.

\begin{definition}[Causal Rules and Causal Theories]
     A \textbf{causal rule} $R$ is an expression of the form 
     $
     \phi \Rightarrow \psi
     $ 
     for two propositional formulas $\phi$ and $\psi$, where~${\cause(R) := \phi}$ is the \textbf{cause} and ${\effect(R) :=  \psi}$ is the \textbf{effect} of $R$. A \textbf{default} is a causal rule of the form~${\phi \Rightarrow \phi}$. Lastly, a \textbf{causal theory} $\Delta$ is a set of causal rules. 
     
     Denote 
     by~${(\Rrightarrow_{\Delta})/2}$ the smallest causal production inference relation such that $\phi \Rrightarrow_{\Delta} \psi$ whenever $\phi \Rightarrow \psi \in \Delta$ and by~$ \mathcal{C}_{\Delta}$ the corresponding consequence operator. Observe that~$\phi \Rrightarrow_{\Delta} \psi$ if and only if $\phi \Rightarrow \psi$ follows from $\Delta$ with the rules  (Strengthening), \mbox{(Weakening)}, (And), \mbox{(Truth and Falsity)}, (Cut) and~(Or) in Definitions \ref{definition - production inference relation}, \ref{definition - regular production infernece relation}, and \ref{definition - basic prodoction inference relation}, i.e., all rules that apply for the      implication in propositional calculus except reflexivity~${\phi \rightarrow \phi}$. A \textbf{causal world} of~$\Delta$ is a causal world of the production inference relation~${(\Rrightarrow_{\Delta})/2}$. Finally, write~${\causal (\Delta) := \causal (\Rrightarrow_{\Delta})}$ for the \textbf{causal worlds semantics} of $\Delta$.
     \label{definition - causal theory}
\end{definition} 

\begin{remark}  
    For any set of propositional formulas $\Phi$, the set $\mathcal{C}_{\Delta}(\Phi)$ consists of all formulas~$\psi$ such that $\Phi \Rightarrow \psi$ can be derived from $\Delta$ using the rules of (Strengthening), (Weakening), (And), (Truth and Falsity), (Cut), and~(Or).
\end{remark}

\textit{Causal Foundation} in Principle \ref{principle - causal foundation} identifies a set of external premises $\mathcal{E}$ that do not require demonstration. In Section 4.5.1 of \cite{Bochman}, he interprets \textit{causal foundation} as the assertion that these external premises~$\epsilon \in \mathcal{E}$ yield defaults in the corresponding causal theory $\Delta$. Overall, we conclude that he applies the following representation:

\begin{Language}
     According to Principle \ref{principle - causal rules}, a causal theory $\Delta$ represents the causal knowledge within an area of science. In particular, $\phi \Rightarrow \psi \in \Delta$ if $\phi$ is a direct cause of~$\psi$, i.e., there exists a demonstration for $\psi$ with premise $\phi$, and the external premises~${\epsilon \in \mathcal{E}}$ yield defaults, i.e., ${\epsilon \Rightarrow \epsilon \in \Delta}$. 

     \cite{Bochman} interprets $\phi \Rrightarrow_{\Delta} \psi$ as $\phi$ explaining $\psi$, i.e., \mbox{knowledge-\textit{that}}~$\phi$ holds explains knowledge-\textit{why} $\psi$ holds. 

     Deviating from \mbox{\cite{Bochman}}, we interpret $\phi \Rrightarrow_{\Delta} \psi$ as stating that there exists a demonstration for $\psi$ with premise $\phi$, i.e., only knowledge-\textit{why} $\phi$ explains knowledge-\textit{why} $\psi$.
     \label{formalization - causal theories}
\end{Language}

\begin{example}
     In the formalism of Example~\ref{example - propositional alphabet}, Language~\ref{formalization - causal theories} models the situation in Example~\ref{example - introduction running example} using the causal theory~$\Delta$ from Example~\ref{example - indemonstrable knowledge}.  
     The causal theory~$\Delta$ gives then rise to the causal worlds $\omega_1$ and $\omega_2$ in Example~\ref{example - causal worlds}.
     \label{example - representation of production inference relation}
\end{example}

\begin{example}
     Consider the following causal theory $\Delta$:
     \begin{align*}
      & ( \textit{rain} \lor \textit{sprinkler} ) \Rightarrow ( \textit{rain} \lor \textit{sprinkler} )  &
      & ( \neg \textit{rain} \land \neg \textit{sprinkle} ) \Rightarrow ( \neg \textit{rain} \land \neg \textit{spr...} ) \\
     & \textit{rain} \lor \textit{sprinkler} \Rightarrow  \textit{wet} && \neg \textit{wet} \Rightarrow \neg \textit{wet} \\ 
      &  \textit{wet} \Rightarrow \textit{slippery}  && \neg \textit{slippery} \Rightarrow \neg \textit{slippery}
     \end{align*}
     Note that $\Delta$ does not make a statement about the proposition $\textit{rain}$. Consequently, the event $\textit{rain}$ cannot be explained by $\Delta$ and therefore $\textit{rain}$ should be false in any causal world of $\Delta$. As the same argument holds also for $\neg \textit{rain}$, there is no causal world of $\Delta$.  Hence, the causal world semantics and \textit{sufficient causation} in Assumption~\ref{principle - Leibniz} are only suitable in cases where the available causal knowledge determines whole worlds precisely.
     \label{example - representation of production inference relation missing literal}
\end{example}

Recall that any world is the deductive closure of its literals and that disjunctions in the causes of a rule can be separated into distinct causal rules by Property~(Or) in Definition~\ref{definition - basic prodoction inference relation}. Consequently, by applying \textit{sufficient causation} in Assumption~\ref{principle - Leibniz}, the analysis may be restricted to determinate causal theories.

\begin{definition}[Literal, Atomic and Determinate Causal Theory]
     A \textbf{literal} causal rule is a causal rule of the form $b_1 \land ... \land b_n \Rightarrow l$ for literals $b_1,...,b_n,l$. If, in addition, $l \in \mathfrak{P}$ is an atom, we call the rule \textbf{atomic}. Furthermore, a \textbf{constraint} is a causal rule~${b_1 \land ... \land b_n \Rightarrow \bot}$ for literals $b_1,...,b_n$.

     Now, a causal theory $\Delta$ is called \textbf{literal} or \textbf{atomic} if it only mentions literal or atomic causal rules. A \textbf{determinate} causal theory $\Delta \cup \textbf{C}$ is the union of a literal causal theory $\Delta$ and a set of constraints $\textbf{C}$. We further say that $\Delta \cup \textbf{C}$ is \textbf{atomic determinate} if the causal theory $\Delta$ is atomic. Lastly, a literal $l$ is a \textbf{default} of a determinate causal theory~$\Delta$ if $l \Rightarrow l \in \Delta$.

     The \textbf{causal structure} \( \graph (\Delta) \) of a literal causal theory \( \Delta \) is the directed graph on the alphabet $\mathfrak{P}$ where an edge \( p \to q \) is drawn if and only if there exists a causal rule of the form  
    $
     b_1 \land \dots \land (\neg) p \land \dots \land b_n \Rightarrow (\neg) q \in \Delta.
    $
 \label{definition - determinate causal theories}
\end{definition}

\begin{remark}
     \cite{Bochman} refers to atomic causal rules and theories as positive literal causal rules and theories, respectively. He uses the term positive determinate causal theory for an atomic determinate causal theory in our sense. 
\end{remark}

Upon committing to \textit{causal rules} in Principle~\ref{principle - causal rules}, \textit{natural necessity} in Principle~\ref{principle - Aquinas}, and \textit{sufficient causation} in Assumption~\ref{principle - Leibniz}, \cite{Bochman} adopts the following approach:

\begin{Language}
    Causal knowledge, which underpins explainability as captured by causal production inference relations that satisfy \emph{natural necessity} in Principle~\ref{principle - Aquinas} and \emph{sufficient causation} in Assumption~\ref{principle - Leibniz}, is expressed in the form of determinate causal theories of Definition~\ref{definition - determinate causal theories}.
    
    \label{language - determinate causal theories}
    \label{formalization - determinate causal theories}
\end{Language}

He then obtains the following characterization for the causal worlds of a determinate causal theory.

\begin{definition}[Completion of a Determinate Causal Theory]
     The \textbf{completion} $\comp (\Delta)$ of a determinate causal theory $\Delta$ is the set of all propositional formulas 
     $
         \displaystyle
         l \leftrightarrow \bigvee_{\phi \Rightarrow l \in \Delta} \phi
     $,
     where $l$ is a literal or $\bot$.
     \label{definition - completion of a determinate causal theory}
\end{definition}

\begin{theorem}[\cite{bochman2005}, Theorem 8.115]
     The causal world semantics~$\causal (\Delta)$ of a determinate causal theory $\Delta$ coincides with the set of all models of its completion, i.e.
     $
     \causal (\Delta) := \{ \omega \text{ world: } \omega \models \comp (\Delta) \}
     \text{. } \square
     $ 
     \label{theorem - completion of a causal theory}
\end{theorem}



Assume that the production inference relation $(\Rrightarrow)/2$ satisfies \textit{sufficient causation} in Assumption~\ref{principle - Leibniz} and expresses complete causal knowledge, thereby determining a set of causal worlds. Let $\omega$ be a world such that $\omega \not\Rrightarrow p$ for some proposition $p \in \mathfrak{P}$. In this case, $\omega$ is  either  a causal world, meaning~${\omega \Rrightarrow \neg p}$, or~$\omega$ is not a causal world, meaning $\omega \Rrightarrow \bot$. 

Thus, the causal world semantics of $(\Rrightarrow)/2$ can be characterized through the negative completion of an atomic determinate causal theory.

\begin{definition}[Negative Completion and Default Negation]
    The \textbf{negative completion}~$\Delta^{nc}$ of the atomic determinate causal theory $\Delta$ is given by
     $$
        \Delta^{nc} := \Delta \cup  \{ \neg p \Rightarrow \neg p \text{ : } p \in \mathfrak{P} \}.
     $$
     A causal theory has \textbf{default negation} if it is the negative completion of all its atomic causal rules and constraints. 
     \label{definition - negative completion and default negation}
\end{definition}

Restricting attention to causal theories with default negation amounts to treating negations as self-evident priors. This reflects the modeling assumption that parameters possess a default state, which -- without loss of generality -- may be taken as \textit{false}. In this way, the dynamic nature of causality is captured by explaining how values deviate from their defaults.

\begin{example}
    In Example~\ref{example - houses introduction}, houses typically do not burn; that is, only a fire requires an explanation.

    To model such scenarios, one begins with atomic causal rules that identify the direct causes of each proposition~$p$. For instance, in Example~\ref{example - introduction running example}, one may posit~${\textit{rain} \Rightarrow \textit{wet}}$ and ${\textit{sprinkler} \Rightarrow \textit{wet}}$ to express that either rain or the sprinkler can cause the road to be wet. If these atomic causal rules do not explain the proposition $p$, this is interpreted as an explanation for the falsity of~$p$, that is, for~$\neg p$. Example~\ref{example - indemonstrable knowledge} illustrates this principle by modeling Example~\ref{example - introduction running example} via the causal theory with default negation~$\Delta$.

    \label{example - defaults}
\end{example}

 To summarize, causal theories with default negation implement the following modeling assumption: 

\begin{assumption}[Default Negation]
    Every negative literal~$\neg p$ is an external premise of the area of science under consideration, i.e., $\neg p \in \mathcal{E}$.
    \label{principle - default negation}
    \label{assumption - default negation}
\end{assumption}

Committing to Bochman's version of Language  \ref{formalization - causal theories} and Language \ref{language - determinate causal theories}, Assumption \ref{assumption - default negation} is expressed as follows:

\begin{formalization}
    Assumption~\ref{assumption - default negation} means that, the given area of science yields a causal theory with default negation.
    \label{formalization - default negation}
\end{formalization}

\subsection{Analysis and Critique of Bochman's Logical Theory of Causality}
\label{subsec: Critique of Bochman's Logical Theory of Causality}

In Language \ref{language - propositional logic} and \ref{language - causal reasoning}, \cite{Bochman} decides to represent explainability as a binary relation $(\Rrightarrow)/2$ on formulas in a propositional alphabet $\mathfrak{P}$, which represent knowledge-\textit{that}.  \mbox{Formalizations \ref{formalization - consistency with deduction} and \ref{formalization - causal sufficiency}}, expressing Principles \ref{principle - consistency with logical reasoning}, \ref{principle - Aquinas},  and Assumption \ref{principle - Leibniz}, yield that $(\Rrightarrow)/2$ is a causal production inference relation. Formalization \ref{formalization - Leibniz}, expressing Principle \ref{principle - Aquinas} and Assumption \ref{principle - Leibniz}, further yields that the corresponding knowledge-\textit{why} is captured in the resulting causal worlds. 

Committing to Principle \ref{principle - causal rules} and Languages~\ref{formalization - causal theories}, and~\ref{language - determinate causal theories}, \cite{Bochman}  represents explainability as determinate causal theories. Finally, Formalization \ref{formalization - default negation}, expressing \textit{default negation} in Assumption \ref{principle - default negation}, yields that~$\Delta$ is a causal theory with default negation. Overall, we showed the following theorem:

\begin{theorem}
    Applying the choices in Language \ref{language - propositional logic}, \ref{language - causal reasoning}, \ref{formalization - causal theories} and \ref{language - determinate causal theories} as well as Formalization \ref{formalization - causal sufficiency}, expressing Principles~ \ref{principle - consistency with logical reasoning}, \ref{principle - causal rules},~\ref{principle - Aquinas}, and Assumption \ref{principle - Leibniz}, yield that explainability gives rise to a causal production inference relation~$(\Rrightarrow_{\Delta})/2$, which is represented by a determinate causal theory $\Delta$. 
    
    The possible states of \mbox{knowledge-\emph{why}} are represented by the causal world semantics $\Causal (\Delta)$. Finally, applying Formalization \ref{formalization - default negation}, expressing  Assumption~\ref{principle - default negation}, leads to a causal theory~$\Delta$ with default negation.~$\square$
    \label{theorem - completeness of causal theories}
\end{theorem}


We identify two potential limitations of Bochman's theory: First, Principle~\ref{principle - causal foundation}, as formalized in Formalization~\ref{formalization - binary semantics}, fails in the presence of cyclic causal relations. Second, causal rules with compound effects may give rise to problematic cases.

\subsubsection{Cyclic Causal Relations}
\label{subsubsec: Cyclic Causal Relations}

Let us take a closer look at cyclic causal relations in the context of the notion of knowledge-\textit{why} proposed by \mbox{\cite{Bochman}}.

\begin{example}
    According to Theorem~\ref{theorem - completion of a causal theory}, the causal theory $\Delta$ in Example \ref{example - counterexample Bochman introduction},  has the causal worlds:
    \begin{align*}
        & \omega_1 := \emptyset, &&
         \omega_3 := \{ \textit{start\_fire}(h_1),~ \textit{fire}(h_1), \textit{fire}(h_2) \},  \\ 
        & \omega_2 := \{ \textit{fire}(h_1),~ \textit{fire}(h_2) \}
        && \omega_4 := \{ \textit{start\_fire}(h_2),~ \textit{fire}(h_1),~ \textit{fire}(h_2) \}, \\
        &&& \omega_5 := \{ \textit{start\_fire}(h_1), \textit{start\_fire}(h_2),\textit{fire}(h_1), \textit{fire}(h_2) \}.
    \end{align*}

    In the causal world $\omega_2$, both Houses $h_1$ and $h_2$ catch fire even though neither of them started burning. This contradicts everyday causal reasoning, as we do not expect houses to catch fire because they potentially influence each other. 

    In particular, we observe that (Strengthening) and (Cut) in Definitions~\ref{definition - production inference relation} and~\ref{definition - regular production infernece relation} imply that $\textit{fire}(h_1) \Rrightarrow_{\Delta} \textit{fire}(h_1)$, even though $\textit{fire}(h_1)$ cannot be demonstrated from the external premises in $\neg \textit{start\_fire}(h_1),~\neg \textit{start\_fire}(h_2) \in \omega_2$. 

    This contradicts Principle \ref{principle - causal foundation}, as expressed in Formalization~\ref{formalization - binary semantics}, if we adopt Bochman's version of Language~\ref{formalization - causal theories}.
    \label{example - counterexample Bochman}
\end{example}

Example~\ref{example - counterexample Bochman} illustrates a drawback of the approach in~\cite{Bochman}, where Principle~\ref{principle - causal foundation}, as expressed in Formalization~\ref{formalization - binary semantics}, fails in the presence of cyclic causal relations, leading to circular ``explanations'' and counterintuitive~results.

\subsubsection{Compound Effects}
\label{subsubsec: Compund Effects}

 Aristotle did not study causal relations involving disjunctions or implications in the effect. In particular, it is unclear what it means to have a \emph{demonstration} of a (logical) implication.  As the following example illustrates, the approach in \cite{Bochman} leads to ``demonstrations'' that allow conclusions to be drawn against the direction of cause and effect:

\begin{example}
    Let $\Delta$ be a causal theory consisting of the causal rules:
    \begin{align*}
        & a \Rightarrow a, && a \Rightarrow (a \rightarrow b), && b \Rightarrow a, 
        && \neg a \Rightarrow \neg a, && \neg b \Rightarrow \neg b.
    \end{align*}
    In this case, (And) in Definition~\ref{definition - production inference relation} yields $a \Rrightarrow_{\Delta} a \land (a \rightarrow b)$, and \mbox{(Weakening)} in Definition~\ref{definition - production inference relation} yields $a \Rrightarrow_{\Delta} b$. This seems problematic, as the ``demonstration'' of~$b$ with premise $a$ relies on the implication ``$a \rightarrow b$'', which contradicts the causal direction.
    \label{example - implication in effects}
\end{example}

Since it is unclear what it means for an implication to be caused or whether entailment $(\vdash)/2$ in classical propositional logic is the appropriate choice in the (Weakening) axiom of Definition~\ref{definition - production inference relation}, Example~\ref{example - implication in effects} highlights a potential issue with general compound effects.

\subsection{Causal Systems: A Generic Representation of Causal Reasoning}
\label{subsec: Causal Systems}

Fix a propositional alphabet $\mathfrak{P}$. To address the issues raised in Remark~\ref{remark - observations} and Section~\ref{subsec: Critique of Bochman's Logical Theory of Causality}, we propose the notion of a deterministic causal system:

\begin{definition}[Causal System]
      A \textbf{(deterministic) causal system} is a tuple $\textbf{CS} := (\Delta, \mathcal{E}, \mathcal{O})$, where:
     \begin{itemize}
         \item $\Delta$ is a literal causal theory called the \textbf{causal knowledge} of $\textbf{CS}$.
         \item $\mathcal{E}$ is a set of literals called the \textbf{external premises} of $\textbf{CS}$.
         \item $\mathcal{O}$ is a set of formulas called the \textbf{observations} of $\textbf{CS}$.
     \end{itemize}
      The causal system $\textbf{CS}$ is \textbf{without observations} if $\mathcal{O} = \emptyset$. Otherwise, the causal system $\textbf{CS}$ \textbf{observes something}. It applies \textbf{default negation}  if every negative literal $\neg p$ for $p \in \mathfrak{P}$ is an external premise, i.e.,~${\neg p \in \mathcal{E}}$ and no external premise is an effect of a causal rule in $\Delta$.

      The \textbf{causal structure} of $\textbf{CS}$ is given by $\graph(\textbf{CS}) := \graph(\Delta)$.
     \label{definition - abductive causal theory}
     \label{definition - causal system}
\end{definition}

\begin{example}
     Let $\textbf{CS}_1 := (\Delta, \mathcal{E}, \emptyset)$ be the causal system without observations in Example \ref{example - introducing causal systems}. 
     If we additionally observe a fire in House $h_1$, i.e.,~${\mathcal{O} := \{ \textit{fire}(h_1) \}}$, this yields the causal system  
     $
     \textbf{CS}_2 := \left( \Delta, \mathcal{E}, \{ \textit{fire}(h_1) \} \right)
     $. Note that both causal systems $\textbf{CS}_1$ and $\textbf{CS}_2$ apply default negation.
     \label{example - causal system}
\end{example}

We use Definition \ref{definition - causal system} together with the following guideline:

\begin{Language}
    \emph{Causal foundation} in Principle~\ref{principle - causal foundation} gives rise to a set of external premises~$\mathcal{E}$ that do not require further explanation. According to \emph{causal rules} in Principle~\ref{principle - causal rules}, causal knowledge is captured by a causal theory~$\Delta$, which contains a causal rule~${\phi \Rightarrow \psi}$ whenever the formula~$\phi$ is a direct cause of the formula~$\psi$, i.e., there exists a \emph{demonstration} of~$\psi$ from the premise~$\phi$. All observations are formalized as a set of formulas~$\mathcal{O}$, yielding the causal system~${\mathbf{CS} := (\Delta, \mathcal{E}, \mathcal{O})}$.
    \label{language - causal systems}
    \label{formalization - causal systems}
\end{Language}

Hereby, we address the concerns in Remark \ref{remark - observations} and Section \ref{subsubsec: Compund Effects} by committing to the following assumtion:

\begin{assumption}
    The causal theory $\Delta$ in Language \ref{formalization - causal systems} is literal.
    \label{assumption - literal causal theory}
\end{assumption}

According to \textit{causal foundation} in Principle \ref{principle - causal foundation}, causal explanations or \emph{demonstrations} should start with \textit{external premises} in $\mathcal{E}$. 

\begin{definition}[Semantics of Causal Systems]
     Let ${\mathbf{CS} := (\Delta, \mathcal{E}, \mathcal{O})}$ be a causal system. The \textbf{explanatory closure} of~$\mathbf{CS}$ is the causal theory 
     $$
     \Delta (\mathbf{CS}) := \Delta \cup \{ l \Rightarrow l \mid l \in \mathcal{E} \}.
     $$
     The \textbf{consequence operator}~$\mathcal{C}$ of~$\mathbf{CS}$ is the consequence operator of the explanatory closure~$\Delta (\mathbf{CS})$.
     
     A \textbf{causal world}~$\omega$ is a world that satisfies~$\mathcal{C} (\omega \cap \mathcal{E}) = \omega$ and~$\omega \models \mathcal{O}$. The set of all causal worlds~$\Causal(\mathbf{CS})$ is called the \textbf{causal world semantics} of~$\mathbf{CS}$. 

    For a formula~$\phi$, the system~$\mathbf{CS}$ has \textbf{knowledge-\textit{that}}~$\phi$, denoted~$\mathbf{CS} \stackrel{\textit{that}}{\models} \phi$, if~$\phi \in \omega$ for all causal worlds~$\omega \in \Causal(\mathbf{CS})$.

     \label{definition - semantics of causal systems}
\end{definition}

\begin{example}
    In the situation of Examples~\ref{example - counterexample Bochman} and~\ref{example - causal system}, we find that the causal system~$\mathbf{CS}_1 := (\Delta, \mathcal{E}, \emptyset)$ has the causal world semantics:
    $$
    \Causal (\mathbf{CS}_1) = \{ \omega_1,  \omega_3, \omega_4, \omega_5 \}.
    $$
    
    Applying \textit{causal foundation} in Principle \ref{principle - causal foundation} and \textit{sufficient causation} in Assumption~\ref{principle - Leibniz}, the system~$\mathbf{CS}_1$ assumes that the causal production inference relation~${(\Rrightarrow_{\Delta(\textbf{CS})})/2}$ explains the occurrence of every possible event based on premises in~$\mathcal{E}$. Since it cannot explain \textit{why}~$\textit{fire} (h_i) \in \omega_2$, it refutes~$\omega_2$, meaning that~$\omega_2$ is not a causal world. 
    
    The causal system~$\mathbf{CS}_2 := (\Delta, \mathcal{E}, \{ \textit{fire} (h_1) \})$ additionally observes a fire in House~$h_1$ and therefore refutes the world~$\omega_1$. It has the causal world semantics:
    $$
    \Causal (\mathbf{CS}_2) = \Causal (\mathbf{CS}_1) \setminus \{ \omega_1 \}.
    $$
    
    \label{example - semantics of causal systems}
\end{example}

Fix an area of science with observations that is captured in a causal system~$\textbf{CS} := (\Delta, \mathcal{E}, \mathcal{O})$ according to Language~\ref{formalization - causal systems}. Note that the explanatory closure $\Delta (\textbf{CS})$ is the causal theory obtained by Language~\ref{formalization - causal theories} and ${\phi \Rrightarrow_{\Delta(\textbf{CS})}\psi}$ means that knowledge-\textit{why} $\phi$ yields knowledge-\textit{why} $\psi$. From Section~\ref{subsubsec: Cyclic Causal Relations}, we conclude that, in general, $\Delta (\textbf{CS})$ yields a causal production inference relation that does not satisfy Principle \ref{principle - causal foundation} as stated in Formalization~\ref{formalization - binary semantics} and, therefore, results in too many causal worlds.

In Definition~\ref{definition - semantics of causal systems}, we enforce Principle \ref{principle - causal foundation} by requiring that each causal world~$\omega$ is fully explained by the external propositions in~$\omega \cap \mathcal{E}$, i.e.,~${\mathcal{C}(\omega \cap \mathcal{E}) = \omega}$. According to \textit{natural necessity} in Principle \ref{principle - Aquinas} and \textit{sufficient causation} in Assumption~\ref{principle - Leibniz} as stated in Formalization \ref{formalization - Leibniz}, explainability~$(\Rrightarrow)/2$ is uniquely determined by its causal worlds. We conclude that Definition~\ref{definition - semantics of causal systems} provides the correct formalization of causal explainability within the given area of science. Finally, we note that the given area of science satisfies \textit{default negation} in Assumption \ref{assumption - default negation} if and only if the causal system $\textbf{CS}$ applies default negation.

\begin{formalization}
    Apply Language \ref{language - causal systems} together with Assumption~\ref{assumption - literal causal theory} to express an area of science with observations in a causal system
    $
    {\textbf{CS} := (\Delta, \mathcal{E}, \mathcal{O})}.
    $ 
    Furthermore, apply Languages \ref{language - propositional logic}, \ref{language - causal reasoning} and Formalization \ref{formalization - consistency with deduction}, expressing Principle~\ref{principle - consistency with logical reasoning}, to represent explainability by a production inference relation $(\Rrightarrow)/2$. Explainability~$(\Rrightarrow)/2$ then satisfies Principles~\ref{principle - causal foundation} and~\ref{principle - Aquinas}, as well as Assumptions~\ref{principle - Leibniz} if and only if it is determined by the causal world semantics of $\textbf{CS}$ as indicated in Formalization \ref{formalization - binary semantics}. 
    Finally, Assumption \ref{assumption - default negation} is satisfied if and only if the causal system $\textbf{CS}$ applies default negation.
    \label{formalization - semantics of causal systems}
\end{formalization}

By \emph{directionality} in Principle~\ref{principle - directionality}, a causal system~$\mathbf{CS}$ acquires \mbox{knowledge-\emph{why}} only if its explanations follow the direction of cause and effect. This, we argue, entails \emph{causal irrelevance} in Principle~\ref{principle - causal irrelevance}, formalized by~\cite{rueckschloß2025rulesrepresentcausalknowledge} as follows:

\begin{formalization}[\cite{rueckschloß2025rulesrepresentcausalknowledge}]
    Let $\textbf{CS} := (\Delta, \mathcal{E}, \mathcal{O})$ be a causal system.  
    For a set of propositions ${S \subseteq \mathfrak{P}}$, let~$\mathfrak{P}^{> S}$ denote the set of all propositions $q \in \mathfrak{P} \setminus S$ that are descendants in $G := \graph(\textbf{CS})$ of some proposition in~$S$ and set $\mathfrak{P}^{\geq S} = \mathfrak{P}^{> S} \cup  S$, $\mathfrak{P}^{< S} = \mathfrak{P} \setminus \mathfrak{P}^{\geq S}$, and $\mathfrak{P}^{\leq S} = \mathfrak{P} \setminus \mathfrak{P}^{> S}$.  
   
    For $\ast \in \{ <, \leq, > , \geq \}$ denote by $\Delta^{\ast S}$ the causal theory of all rules $R \in \Delta$ with ${\effect(R) = (\neg) p}$,~${p \in \mathfrak{P}^{\ast S}}$.
    
    For a $\mathfrak{P}^{\leq S}$-structure $\omega^{\leq S}$ set ${\textbf{CS}^{> S, \omega} := (\Delta^{>S}, \mathcal{E} \cup \{ p, \neg p \text{: } p \in \mathfrak{P}^{\leq S} \}, \omega^{\leq S} )}$ 

    The causal system $\textbf{CS}$ satisfies \emph{causal irrelevance} in Principle~\ref{principle - causal irrelevance} if and only if for every set $S \subseteq \mathfrak{P}$ and every $\mathfrak{P}^{\leq S}$-structure $\omega^{\leq S}$, the system
    $
        \textbf{CS}^{> S, \omega}
    $
    has at least one causal world; that is, it is not possible to falsify $\omega^{\leq S}$ with the causal knowledge in $\Delta^{>S}$.
    \label{formalization - causal irrelevance - deterministic}
\end{formalization}

\cite{Williamson2001} proposes Principle~\ref{principle - causal irrelevance} in the context of Bayesian networks as a weakening of the Markov assumption \mbox{\citep{Causality}}. 
Accordingly, Formalization \ref{formalization - causal irrelevance - deterministic} could be viewed as a deterministic analogue of the Markov assumption. Our considerations motivate the following definition.

\begin{definition}[Knowledge-\emph{Why}]
     A causal system $\textbf{CS} := (\Delta, \mathcal{E}, \mathcal{O})$ \textbf{provides demonstrations} if it satisfies Principle \ref{principle - causal irrelevance} according to Formalization~\ref{formalization - causal irrelevance - deterministic}.
     
     Let $\phi$ be a formula. If~${\textbf{CS} := (\Delta, \mathcal{E}, \mathcal{O})}$ provides demonstrations, it has \textbf{knowledge-\textit{why}}~$\phi$, written~$\textbf{CS}\stackrel{\textit{why}}{\models} \phi$, if~$\textbf{CS}\stackrel{\textit{that}}{\models} \phi$ and ${(\Delta, \mathcal{E}, \mathcal{O} \cap \mathcal{E}) \stackrel{\textit{that}}{\models} \phi}$. Here, $\mathcal{O} \cap \mathcal{E}$ denotes the set of all formulas $o \in \mathcal{O}$ in the external premises $\epsilon \in \mathcal{E}$, i.e., no observations are needed to conclude against the causal direction.
     \label{definition - knowledge-why}
\end{definition}

Within causal systems providing demonstrations, we propose the following formalization of \emph{directionality} in  Principle \ref{principle - directionality}.

\begin{formalization}
    Assume that the causal system $\textbf{CS} := (\Delta, \mathcal{E}, \mathcal{O})$ in Formalization \ref{formalization - causal irrelevance - deterministic} provides demonstrations, i.e., it satisfies \emph{causal irrelevance} in Principle \ref{principle - causal irrelevance} and let $\phi$ be a formula such that $\textbf{CS} \stackrel{\text{that}}{\models} \phi$. 
    
    The explanation of $\textbf{CS}$ for $\phi$ satisfies \emph{directionality} in Principle \ref{principle - directionality} if and only if $(\Delta, \mathcal{E}, \mathcal{O} \cap \mathcal{E}) \stackrel{\textit{that}}{\models} \phi$ in Definition~\ref{definition - knowledge-why}.
    Consequently, the system possesses knowledge-\emph{that} and knowledge-\emph{why} as indicated in Definitions \ref{definition - causal system} and \ref{definition - knowledge-why}.
    \label{formalization - knowledge-why}
\end{formalization}

\subsection{Interpreting Pearl's Structural Causal Models as Causal Systems}
\label{subsec: Interpreting Artificial Intelligence Frameworks as Deterministic Causal Systems}

Finally, we show how the structural causal models of \cite{Causality} can be interpreted as causal systems. This allows us to apply Language \ref{language - causal systems}, Formalizations \ref{formalization - semantics of causal systems}, and \ref{formalization - knowledge-why} to evaluate the kind of knowledge provided by this formalism. 

\subsubsection{Pearl's Functional Causal Models}
\label{subsubsec: Pearl's Functional Causal Models}

\cite{Causality} suggests modeling causal relationships with deterministic functions. This leads to the following definition of structural causal models. 

\begin{definition}[Structural Causal Model $\text{\cite[§7.1.1]{Causality}}$]
A \textbf{(Boolean) (structural) causal model} 
$
{\mathcal{M} := (\textbf{U}, \textbf{V}, \Error, \Pa, \textbf{F})},
$
is a tuple, where
\begin{itemize}
    \item $\textbf{U}$ is a finite set of \textbf{external} variables representing the part of the world outside the model
    \item $\textbf{V}$ is a finite set of \textbf{internal} variables determined by the causal relationships in the model
    \item $\Error(\cdot)$ is a function assigning to each internal variable $V \in \textbf{V}$ its \textbf{error terms} $\Error (V) \subseteq \textbf{U}$, i.e.~the external variables $V$ directly depends on  
    \item $\Pa(\cdot)$ is a function assigning to each internal variable $V \in \textbf{V}$ its \textbf{parents}~${\Pa (V) \subseteq \textbf{V}}$, i.e.~the set of internal variables $V$ directly depends on  
    \item $\textbf{F}(\cdot)$ is a function assigning to every internal variable $V \in \textbf{V}$ a map
    $$
    \textbf{F}(V) := F_V  :  \{ \textit{True},~\textit{False} \}^{\Pa (V)} \times  \{ \textit{True},~\textit{False} \}^{\Error(V)} \rightarrow \{ \textit{True},~\textit{False} \},
    $$
    which itself assigns to each value assignments $\pa (V)$ and $\error (V)$ of the parents $\Pa (V)$ and  the error terms $\Error (V)$, respectively, a value 
    $$
    F_V(\pa(V), \error (V)) \in \{ \textit{True},~\textit{False} \}.
    $$
\end{itemize}
Here, for a subset of variables $\textbf{X} \subseteq \textbf{U} \cup \textbf{V}$, a \textbf{value assignment} is a function~${\textbf{x} : \textbf{X} \rightarrow \{ \textit{True}, \textit{False} \}}$. A \textbf{situation} is a value assignment~$\textbf{u}$ for the external variables $\textbf{U}$. Finally, $\mathcal{M}$  is identified with the system of \textbf{structural equations}
$$
\mathcal{M} := \{ V := F_V(\Pa (V), \Error(V)) \}_{ V \in \textbf{V} }.
$$ 
A \textbf{solution} $\textbf{s}$ of $\mathcal{M}$ then is a value assignment on the variables $\textbf{U} \cup \textbf{V}$ such that each equation in~$\mathcal{M}$ is satisfied.

To a structural causal model $\mathcal{M}$ one associates its \textbf{causal diagram} or \textbf{causal structure} $\graph(\mathcal{M})$, which is the directed graph on the internal variables $\textbf{V}$ obtained by drawing an edge $p \rightarrow q$ if and only if $p \in \Pa (q)$. The model~$\mathcal{M}$ is \textbf{acyclic} if its causal structure $\graph(\mathcal{M})$ is a directed acyclic graph.
\label{definition - structural causal model}
\end{definition}

\begin{notation}
    The parents~$\Pa(V)$ and error terms~$\Error(V)$ of an internal variable~$V \in \textbf{V}$ are typically evident from its defining function~$F_V$. Accordingly, when specifying a causal model, this work omits explicit mention of the parent map~$\Pa(\cdot)$ and the error term map~$\Error(\cdot)$.
\end{notation}

\begin{example}
The situation in Example \ref{example - introduction running example} is represented by a structural causal model~$\mathcal{M}$ with internal variables $\textbf{V} := \{ \textit{rain},~\textit{sprinkler},~\textit{wet},~\textit{slippery} \}$, external variables~${\textbf{U} := \{ \textit{cloudy} \}}$ and with functions, given by Equations (\ref{equation - structural equations sprinkler}).

One finds for instance that $\Pa (\textit{wet}) = \{ \textit{sprinkler} \}$ and~${\Error (\textit{rain}) = \{ \textit{cloudy} \}}$. The causal structure $\graph(\mathcal{M})$ of $\mathcal{M}$ is obtained from Graph (\ref{equation - causal structure}) by erasing the node $\emph{cloudy}$ together with all outgoing arrows.
Hence, $\mathcal{M}$ is an acyclic causal model.
\label{example - sprinkler}
\label{example - sprinkler deterministic}
\end{example}

Structural causal models are of interest because they can represent the effects of external interventions. According to Chapter 7 of \cite{Causality}, the key idea is that the modified model $\mathcal{M}_{\textbf{i}}$ represents the minimal change to a model $\mathcal{M}$ necessary to enforce the values specified by $\textbf{i}$:

\begin{definition}[Modified Causal Model]
    Let~$
    {\mathcal{M} := (\textbf{U}, \textbf{V}, \Error, \Pa, \textbf{F})}
    $ be a structural causal model.  
    Given a subset of internal variables $\textbf{I} \subseteq \textbf{V}$ with a value assignment $\textbf{i}$, the \textbf{modified (causal) model} or \textbf{submodel} is defined as:
    $$
    \mathcal{M}_{\textbf{i}} := (\textbf{U}, \textbf{V}, \Error, \Pa, \textbf{F}_{\textbf{i}}).
    $$
    In particular, the function $\textbf{F}$ is replaced with $\textbf{F}_{\textbf{i}}$, which is given by setting
    $$
    \textbf{F}_{\textbf{i}} (V) (\pa(V), \error(V)) :=    
    \begin{cases}
        \textbf{i}(V), & \text{if } V \in \textbf{I}, \\
        \textbf{F} (V) (\pa(V), \error(V)), & \text{otherwise}.
    \end{cases}
    $$
    for every internal variable $V \in \textbf{V}$, where $\pa(V)$ and $\error(V)$ denote value assignments for the parents~$\Pa(V)$ and the error terms $\Error(V)$, respectively.
    \label{definition - modified causal model}
\end{definition}

\begin{notation}
    Let $V \in \textbf{V}$ be an internal variable of a structural causal model~${\mathcal{M}}$. In this case, one writes $\mathcal{M}_V := \mathcal{M}_{V := \textit{True}}$ and $\mathcal{M}_{\neg V} := \mathcal{M}_{V := \textit{False}}$.
\end{notation}

\begin{example}
    Switching the sprinkler off in the model of \mbox{Example~\ref{example - sprinkler deterministic}} yields  
    the modified model with Structural Equations (\ref{equation - structural equations sprinkler after interventions}).
    
    \label{example - sprinkler deterministic intervention}
\end{example}

As in Example~\ref{example - sprinkler deterministic intervention}, actions often force a variable in a causal model to take on a new value. \cite{Causality} emphasizes that submodels~$\mathcal{M}_{\textbf{i}}$ typically arise from performing actions that set certain variables to specific values, a process formalized by the introduction of the \emph{do}-operator. To obtain well-defined results, he restricts himself to the study of functional causal models: 

\begin{definition}[Functional Causal Model]
    A \textbf{(functional) causal model} is a structural causal model~${\mathcal{M} := (\textbf{U}, \textbf{V}, \textbf{R}, \Error, \Pa, \textbf{F})}$ such that for each value assignment $\textbf{i}$ on a subset of internal variables~${\textbf{I} \subseteq \textbf{V}}$ every situation $\textbf{u}$ of~$\mathcal{M}_{\textbf{i}}$ yields a unique solution~$\textbf{s}_{\textbf{i}} (\textbf{u})$ of the modified model $\mathcal{M}_{\textbf{i}}$. 
    
    \label{definition - functional causal model}
\end{definition}

\begin{remark}
    Acyclic structural causal model are functional causal models.
\end{remark}

\begin{example}
    Reconsider the causal model from Example~\ref{example - sprinkler deterministic} and 
    assume that it is sunny. This corresponds to the situation $\textbf{u}$, where~${\textit{cloudy} = \textit{False}}$. 
    By analyzing the model $\mathcal{M}$ and the modified model $\mathcal{M}_{\neg \textit{sprinkler}}$ from Example~\ref{example - sprinkler deterministic intervention}, one finds that~${\textit{slippery} = \textit{True}}$ in the solution of $\mathcal{M}$, whereas $\textit{slippery} = \textit{False}$ in the solution of the modified model $\mathcal{M}_{\neg \textit{sprinkler}}$. Consequently, the road will become dry if one intervenes by manually switching off the sprinkler. 
\end{example}

\subsubsection{Interpreting Pearl's Structural Causal Models as Causal Systems}
\label{subsubsec: Interpreting Pearl's Functional Causal Models as Causal Systems}

We propose causal systems without observations that apply default negation  
 as a language for the structural causal models of \cite{Causality}.

\begin{definition}[Bochman Transformation]
    The \textbf{Bochman transformation} of a structural causal model~$\mathcal{M} :=  (\textbf{U}, \textbf{V}, \Error, \Pa, \textbf{F})$ is the causal system without observations~$\textbf{CS}(\mathcal{M}) := (\Delta, \mathcal{E}, \emptyset)$, where the causal knowledge is given by~$\Delta := \{ F_V \Rightarrow V \}_{V \in \textbf{V}}$ 
    and the external premises by~${\mathcal{E} := \textbf{U} \cup \{ \neg V \}_{V \in \textbf{U} \cup \textbf{V} }}$.
    \label{definition - Bochman transformation of abductive causal theories}
\end{definition}

\begin{example}
   The Bochman transformation of the causal model~$\mathcal{M}$ from Example~\ref{example - sprinkler deterministic} is the causal system~$\textbf{CS}(\mathcal{M}) := (\Delta, \mathcal{E}, \emptyset)$, where~$\Delta$ is given by Rules~(\ref{equation - first rule sprinkler}) and~(\ref{equation - second rule sprinkler}), and~$\mathcal{E} := \{\textit{cloudy},\neg \textit{cloudy},\neg \textit{sprinkler},\neg \textit{rain},\neg \textit{wet},\neg \textit{slippery}\}$.

\label{example - Bochman transformation}
\end{example}

\begin{remark}
    Without loss of generality, the functions~${F_V(\pa(V), \error(V))}$ may be assumed to be expressed in disjunctive normal form, as illustrated in Equations~(\ref{equation - structural equations sprinkler}). By applying~(Or) from Definition~\ref{definition - basic prodoction inference relation}, the Bochman transformation~$\mathbf{CS}(\mathcal{M})$ can therefore be identified with a causal system that employs default negation, while preserving the set of causal worlds.
\end{remark}


The causal worlds $\omega$ of the Bochman transformation~$\textbf{CS} (\mathcal{M})$ of a causal model $\mathcal{M}$ correspond to solutions of $\mathcal{M}$. 

\begin{theorem}
    If $\mathcal{M}$ is a structural causal model, every causal world $\omega$ of the Bochman transformation~$\textbf{CS} (\mathcal{M})$ yields a solution of $\mathcal{M}$. The converse also holds if   the causal model $\mathcal{M}$ is acyclic. 
    \label{theorem - Bochman transformation for deterministic causal models 1}
\end{theorem}

\begin{proof}
    This follows from Theorem \ref{theorem - completion of a causal theory}, as every possible world of~$\textbf{CS}(\mathcal{M})$ is a model of the completion of the explanatory closure $\Delta (\textbf{CS}(\mathcal{M}))$. 
\end{proof}

Applying Formalization~\ref{formalization - semantics of causal systems} and Theorem \ref{theorem - Bochman transformation for deterministic causal models 1},  
causal systems 
define the feasible solutions of structural causal models that align with Principles~\ref{principle - consistency with logical reasoning}, \mbox{\ref{principle - causal foundation},~\ref{principle - Aquinas}}, and Assumptions~\mbox{\ref{principle - Leibniz}} and~\mbox{\ref{principle - default negation}}. Since the Bochman transformation associates each causal model $\mathcal{M}$ with a causal system without observations, we conclude:
\begin{corollary}
    Acyclic structural causal models represent \mbox{knowledge-\emph{why}}. 
    \label{corollary - causal models represent knowledge-why}
\end{corollary}

In particular, we argue that the Bochman transformation provides the correct extension of the theory of causality in~\cite{Causality} beyond the scope of acyclic models.


\subsection{External Interventions in Causal Systems}
\label{subsubsec: External Interventions in Causal Systems}

Recall that the key idea of modeling an external intervention $\textbf{i}$ is to minimally modify the causal description for a given situation so that $\textbf{i}$ is enforced as true. We propose the following approach to handling external interventions in causal systems, which also accounts for modifications to external premises.

\begin{definition}[Modified Causal Systems]
    Let $\textbf{CS} := (\Delta, \mathcal{E}, \mathcal{O})$ be a causal system, and let $\textbf{i}$ be a value assignment on a set of atoms $\textbf{I} \subseteq \mathfrak{P}$. To represent the intervention of forcing the atoms in $\textbf{I}$ to attain values according to the assignment $\textbf{i}$, we construct the \textbf{modified causal system}
    $$
    \textbf{CS}_{\textbf{i}} := (\Delta_{\textbf{i}}, \mathcal{E}_{\textbf{i}}, \mathcal{O}),
    $$
    which is obtained from $\textbf{CS}$ by applying the following modifications:
    \begin{itemize}
        \item Remove all rules $\phi \Rightarrow p \in \Delta$ and $\phi \Rightarrow \neg p \in \Delta$ for all $p \in \textbf{I}$.
        \item Remove external premises $p \in \mathcal{E}$ and $\neg p \in \mathcal{E}$ if ${p \in \textbf{I}}$.
        \item Add a rule $\top \Rightarrow l$ to $\Delta_{\textbf{i}}$ for all literals ${l \in \textbf{i}}$.
    \end{itemize} 
    \label{definition - modified causal systems}
\end{definition}

\begin{remark}
    According to Remark~\ref{remark - metaphysical justification}, the causal rules of the form $\top \Rightarrow l$ in the modified causal system of Definition~\ref{definition - modified causal systems} require additional justification. This hints at potential issues regarding the interpretation of external interventions, as discussed, for instance, in \cite{Dong_2023}.
\end{remark}

\begin{example}
    Recall the causal system $\textbf{CS} := (\Delta, \mathcal{E}, \emptyset)$ 
    from Example~\ref{example - Bochman transformation}.
    Suppose we switch the sprinkler off, as in Example~\ref{example - sprinkler deterministic intervention}, by intervening according to $\textbf{i} := \{ \neg \textit{sprinkler} \}$. This yields the modified system $\textbf{CS}_{\textbf{i}} := (\Delta_{\textbf{i}},  \mathcal{E}_{\textbf{i}}, \emptyset)$, where:  
    \begin{align*}
            & \Delta_{\textbf{i}} := \{\top \Rightarrow \neg \textit{sprinkler},~ \textit{cloudy} \Rightarrow \textit{rain},~\textit{sprinkler}\lor \textit{rain} \Rightarrow \textit{wet},~ \textit{wet} \Rightarrow \textit{slippery} \}, \\
            & \mathcal{E}_{\textbf{i}} = \{ \textit{cloudy},~ \neg \textit{cloudy},~\neg \textit{rain},~\neg \textit{wet},~\neg \textit{slippery} \}.
    \end{align*}
    Suppose Petrus intervenes and forces the weather to be sunny, i.e., he intervenes according to~${\textbf{i} := \{ \neg \textit{cloudy} \}}$. This yields:  
    \begin{align*}
        & \Delta_{\textbf{i}} := \{\top \Rightarrow \neg \textit{cloudy} \} \cup \Delta, 
        && \mathcal{E}_{\textbf{i}} := \{ \neg \textit{sprinkler},~\neg \textit{rain},~\neg \textit{wet},~\neg \textit{slippery} \}.
    \end{align*}
    \label{example - intervention in an abductive causal theory}
\end{example}

As expected, the concept of intervention, defined in Definition \ref{definition - modified causal systems}, behaves consistently with the Bochman transformation in Definition \ref{definition - Bochman transformation of abductive causal theories}.

\begin{proposition}
    For any structural causal model $\mathcal{M} := (\textbf{U}, \textbf{V}, \Error, \Pa, \textbf{F})$ and any truth value assignment~$\textbf{i}$ on the internal variables ${\textbf{I} \subseteq \textbf{V}}$, 
    the causal systems~${\textbf{CS}(\mathcal{M}_{\textbf{i}})}$ and $\textbf{CS} (\mathcal{M})_{\textbf{i}}$ have the same causal worlds.
    \label{proposition - deterministic Bochman transformation is well-defined}
\end{proposition}

\begin{proof}
    We may, without loss of generality, assume that we intervene on only one variable, i.e., $\textbf{i} := \{ l \}$.

    \begin{case}
       Suppose we have $\textbf{i} = \{ p \}$ for some atom $p \in \mathfrak{P}$.
    \end{case}

    The causal systems $\textbf{CS}(\mathcal{M}_{\textbf{i}})$ and $\textbf{CS}(\mathcal{M})_{\textbf{i}}$ coincide, except that $\textbf{CS}(\mathcal{M}_{\textbf{i}})$ includes the external premise $\neg p$, whereas $\textbf{CS}(\mathcal{M})_{\textbf{i}}$ does not. However, since both systems contain the rule $\top \Rightarrow p$, the external premise $\neg p$ cannot be used to explain any world $\omega$ without leading to a contradiction $\bot$. We conclude that~$\textbf{CS}(\mathcal{M}_{\textbf{i}})$ and $\textbf{CS}(\mathcal{M})_{\textbf{i}}$ have the same causal worlds, as desired.

    \begin{case}
       Suppose we have $\textbf{i} = \{ \neg p \}$ for some atom $p \in \mathfrak{P}$.
    \end{case}

    The causal systems $\textbf{CS}(\mathcal{M}_{\textbf{i}})$ and $\textbf{CS}(\mathcal{M})_{\textbf{i}}$ differ in the following ways:
    \begin{itemize}
        \item The system $\textbf{CS}(\mathcal{M}_{\textbf{i}})$ includes the rule $\bot \Rightarrow p$ and the external premise~$\neg p$.
        \item The system $\textbf{CS}(\mathcal{M})_{\textbf{i}}$ includes the rule $\top \Rightarrow \neg p$ but  no external premise~$\neg p$.
    \end{itemize}
    
    According to Theorem 4.23 in \cite{Bochman}, the rule $\bot \Rightarrow p$ cannot be used to explain a causal world. Therefore, in the absence of an external premise~$p$, the external premise $\neg p$ is equivalent to stating the rule $\top \Rightarrow \neg p$.
\end{proof}

To judge whether a modified causal system~$\textbf{CS}_{\textbf{i}}$ correctly predicts the effect of an intervention~$\textbf{i}$, we rely on \emph{non-interference} in Principle~\ref{assumption - feasibility of external interventions}, which motivates the following definition:

\begin{definition}[Semantics of External Interventions]
    Let $\textbf{CS} := (\Delta, \mathcal{E}, \mathcal{O})$ be a causal system, and ${\textbf{CS}_{\textbf{i}} := (\Delta_{\textbf{i}}, \mathcal{E}_{\textbf{i}}, \mathcal{O})}$ as in Definition \ref{definition - modified causal systems}. The system~$\textbf{CS}$ \textbf{knows} that a formula $\phi$ is true \textbf{after intervening} according to $\textbf{i}$, written~${
    \textbf{CS} \stackrel{\Do (\textbf{i})}{\models} \phi,
    }$
    if and only if 
    $
    \textbf{CS}_{\textbf{i}}\stackrel{\textit{why}}{\models} \phi
    $ 
    and no atom $p \in \textbf{I}$ appears in an observation $o \in \mathcal{O}$.
    \label{definition - semantics of external interventions}
\end{definition}

\begin{remark}
    According to \cite{Causality}, the joint act of intervening and observing generally leads to \textit{counterfactual reasoning}, i.e., reasoning about alternative worlds, which lies beyond the scope of this work. 

\end{remark}

To summarize, we argue for the following result.

\begin{formalization}
    Consider an area of science for which Language~\ref{language - causal systems} yields a causal system~$\textbf{CS}$. \emph{Non-interference} in  Principle~\ref{assumption - feasibility of external interventions}, Definitions~\ref{definition - modified causal systems} and~\ref{definition - semantics of external interventions} correctly characterize the knowledge represented by~$\textbf{CS}$ concerning the effects of external interventions.  
    \label{formalization - external interventions}
\end{formalization}

\subsection{The Constraint and Explanatory Content of Causal Reasoning}
\label{sec: The Constraint and Explanatory Content of Causal Reasoning}

In Section \ref{sec: Knowledge under Uncertainty}, we extend the framework of causal systems to incorporate degrees of belief, represented by probabilities. As a prerequisite for this extension, we reformulate their semantics. Following \cite{Bochman}, we observe that causal theories can be separated into constraint and explanatory components:

\begin{definition}[Constraint and Explanatory Content]
   The \textbf{constraint content} of a causal rule~${R := (\phi \Rightarrow \psi)}$ is the corresponding implication 
   $$
   \constraint(R) := \constraint(\phi \Rightarrow \psi) := (\phi \rightarrow \psi).
   $$
   For a causal theory $\Delta$, the \textbf{constraint content} is defined to be
   $$
   {\constraint(\Delta) := \{ \constraint(R) \text{ : } R \in \Delta \}}
   .$$ 
   The \textbf{explanatory content} of $\Delta$ for a world $\omega$ is the causal theory 
   $$
   \Delta \vert_\omega := \{ R \in \Delta \text{ : } \omega \models \constraint (R) \}.
   $$ 
   Let $\textbf{CS} := (\Delta, \mathcal{E}, \mathcal{O})$ be a causal system.
    The \textbf{constraint content} is defined as ${
    \constraint(\textbf{CS}) := \constraint(\Delta)
    }$,
    the \textbf{explanatory content} is defined by
    ${
    \textbf{CS} \vert_{\omega} := (\Delta \vert_{\omega}, \mathcal{E}, \emptyset)
    }$,
    and $\mathcal{C} \vert_{\omega}$ denotes the corresponding consequence operator.
 
    Let $\omega$ be a world. A formula $\phi$ is \textbf{explainable} in $\omega$, written ${\omega \models \explains (\phi)}$, if~${
    \phi \in \mathcal{C}\vert_{\omega} (\omega \cap \mathcal{E})}$ or  $\neg \phi \in \mathcal{C}\vert_{\omega} (\omega \cap \mathcal{E})
    $. The world $\omega$ satisfies \textbf{(natural) necessity} with respect to $\textbf{CS}$ if $\omega \models \constraint(\textbf{CS})$. It is \textbf{explainable} with respect to~$\textbf{CS}$ if all formulas $\phi \in \omega$ are explainable, i.e., ${\omega \models \explains (\phi)}$ for all formulas $\phi$, or equivalently, ${\omega \models \explains (l)}$ for all literals $l$. 
    
    The event that $\textbf{CS}$ satisfies \textbf{(natural) necessity} is the set
    $$
    \necessary (\textbf{CS}) := \{ \omega \text{ world: } \omega \models \constraint (\textbf{CS}) \}.
    $$
    The event that $\textbf{CS}$ is \textbf{(causally) sufficient} is the set of all explainable worlds, 
    $$
    \sufficient (\textbf{CS}) := \{ \omega \text{ world: } \omega \models \explains (l) \text{ for all literals } l \}.
    $$
    \label{definition - logical and explanatory content}
    \label{definition - causal decidability}
\end{definition}

Recall the following result from Chapter 3 in \cite{Bochman}.

\begin{lemma}
    Stating a causal rule $\phi \Rightarrow \psi$ in a causal theory $\Delta$ is equivalent to stating the \textbf{constraint} $\phi \land \neg \psi \Rightarrow \bot$ and the \textbf{explanatory rule} $ \phi \land \psi \Rightarrow \psi$.~$\square$
    \label{lemma - splitting a causal rule into a constraint and an explanatory component}
\end{lemma}

    


We now give the desired reformulation of the semantics of causal systems.

\begin{proposition}
    Let $\textbf{CS} := (\Delta, \mathcal{E}, \mathcal{O})$ be a causal system.  
    A world $\omega$ is a causal world of $\textbf{CS}$ if and only if 
    $
    \omega \in \necessary (\textbf{CS}) \cap \sufficient (\textbf{CS}) \cap \mathcal{O},
    $
    where the observations $\mathcal{O}$ are identified with the set of all worlds $\omega \models \mathcal{O}$.
    \label{proposition - splitting in logical and explanatory content under causal sufficiency}
\end{proposition}

\begin{proof}
Assume that $\omega$ is a causal world of $\textbf{CS}$. According to Definition \ref{definition - semantics of causal systems}, it follows that $\omega = \mathcal{C} (\omega \cap \mathcal{E})$ and $\omega \models \mathcal{O}$, i.e., $\omega \in \mathcal{O}$. 

Suppose there is a causal rule ${R := (\phi \Rightarrow \psi) \in \Delta}$ such that $\omega \not \models \constraint (R)$, i.e., $\omega \models \phi \land \neg \psi$. According to Lemma \ref{lemma - splitting a causal rule into a constraint and an explanatory component}, we may, without loss of generality, assume that $\phi \land \neg \psi \Rightarrow \bot \in \Delta$. 

Since ${\omega = \mathcal{C} (\omega \cap \mathcal{E})}$, it follows that ${(\omega \cap \mathcal{E}) \Rrightarrow_{\Delta} \phi \land \neg \psi}$. Next, applying~(Cut) in Definition \ref{definition - regular production infernece relation} yields ${(\mathcal{E} \cap \omega) \Rrightarrow_{\Delta} \bot}$ and ${\bot \in \mathcal{C} (\omega \cap \mathcal{E})}$, which contradicts the fact~${\omega = \mathcal{C} (\omega \cap \mathcal{E})}$. Hence, ${\omega \models \constraint(\textbf{CS})}$ and~${\omega \in \necessary(\textbf{CS})}$. 

Since ${\Delta = \Delta \vert_{\omega}}$ and~${\mathcal{C} = \mathcal{C} \vert_{\omega}}$, $\omega$ is explainable with $\textbf{CS}$, i.e.,~${\omega \in \sufficient (\textbf{CS})}$.

Conversely, assume that ${\omega \in \necessary (\textbf{CS}) \cap \sufficient (\textbf{CS}) \cap \mathcal{O}}$. It follows that ${\omega \models \constraint (\textbf{CS})}$, $\omega \models \mathcal{O}$, and ${\mathcal{C}\vert_{\omega} (\omega \cap \mathcal{E}) = \omega}$. Thus, $\Delta \vert_{\omega} = \Delta$ and~${\mathcal{C} = \mathcal{C}\vert_{\omega}}$, concluding that $\omega$ is a causal world.
\end{proof}

Let $\textbf{CS} := (\Delta, \mathcal{E}, \mathcal{O})$ be a causal system. Since the axioms of a causal production inference relation in Definitions \ref{definition - production inference relation}, \ref{definition - regular production infernece relation}, and \ref{definition - basic prodoction inference relation} capture all properties of implication except reflexivity, i.e., $\phi \rightarrow \phi$, we argue for the following result:

\begin{formalization}[Natural Necessity]
    Within a world $\omega$, explainability, as represented by a causal system ${\textbf{CS} := (\Delta, \mathcal{E}, \mathcal{O})}$, satisfies \emph{natural necessity} in Principle~\ref{principle - Aquinas} if and only if ${\omega \models \constraint (\textbf{CS})}$. 

    Thus, the set $\necessary (\textbf{CS})$ consists of all worlds in which \emph{natural necessity} in Principle~\ref{principle - Aquinas} holds.
    \label{formalization - natural necessity}
\end{formalization}

Upon abandoning \textit{natural necessity}, \textit{sufficient causation} in Assumption~\ref{principle - Leibniz} ensures that every world is explainable. We argue for the following result:

\begin{formalization}[Sufficient Causation]
    Within a world $\omega$, explainability, as represented by a causal system ${\textbf{CS} := (\Delta, \mathcal{E}, \mathcal{O})}$, satisfies \emph{sufficient causation} as stated in Assumption~\ref{principle - Leibniz} if and only if ${\omega \models \explains (l)}$ for all literals $l$. 
    
    Thus, the set ${\sufficient (\textbf{CS})}$ consists of all worlds in which \emph{sufficient causation} in Assumption~\ref{principle - Leibniz} holds.
    \label{formalization - sufficient causation}   
\end{formalization}

\section{Knowledge-\emph{Why} under Uncertainty}
\label{sec: Knowledge under Uncertainty}

Since knowledge about the real world typically involves uncertainty, the next goal is to extend areas of science, as represented by causal systems in Section~\ref{sec: Knowledge in Deterministic Systems}, by incorporating degrees of belief, specifically probabilities.

\subsection{Preliminaries}
\label{subsec: Preliminaries for Uncertainty}

As in Section~\ref{sec: Knowledge in Deterministic Systems}, the prerequisites for this endeavor are gathered first.


\subsubsection{Probability Theory and the Principle of Maximum Entropy}
\label{subsubsec: Probability Theory}
\label{subsubsec: The Principle of Maximum Entropy: Deduction under Uncertainty}

This work restricts attention to finite probability spaces and reasons about events using the basic terminology of random variables, conditional probabilities, and independence. An introduction to this material can be found in, for example, Chapter 5 of \cite{Michelucci24}.

Fix a sample space $\Omega$. Informally, the Bayesian viewpoint is adopted, according to which the probability~${\pi(A) \in [0,1]}$ of an event $A \subseteq \Omega$ represents a rational agent's degree of belief in the truth of~$A$. Upon observing an event~$B$, the agent is assumed to update his beliefs by forming the conditional probability
\[
     \pi (A \mid B) := 
     \begin{cases}
     \dfrac{\pi (A \wedge B)}{\pi (B)} & \text{if } \pi (B) \neq 0, \\
     0, & \text{otherwise}
     \end{cases}.
\]

Furthermore, a rational agent extends his beliefs as follows: 

\begin{principle}[Maximum Entropy]
Given probabilities $\pi(A_i) \in [0,1]$ of the events~$A_i \subseteq \Omega$,~${1 \leq i \leq n}$, ${n \in \mathbb{N}_{\geq 1}}$, a rational agent assumes the distribution~$\pi$ on~$\Omega$ that results from maximizing the \textbf{entropy}
$
\displaystyle
H (\pi) := \sum_{\omega \in \Omega} (- \ln (\pi (\omega))) \cdot \pi (\omega)
$
under the constraint that~$A_i$ occurs with probability $\pi (A_i)$.
\label{principle - maximum entropy}
\end{principle}


As shown by \mbox{\cite{Shannon}}, the entropy~$H(\pi)$ quantifies the amount of missing information in a distribution~$\pi$. Given only knowledge about the probabilities~${\pi(A_1), \ldots, \pi(A_n)}$, the principle of \emph{maximum entropy} prescribes selecting the distribution~$\pi$ on~$\Omega$ that is consistent with this information while maximizing entropy, thereby reflecting maximal uncertainty beyond the known constraints.


Assume the random variables in $\mathfrak{P} := \{ p_1,...,p_m \}$ to be Boolaen, i.e, they yield a propositional alphabet. A world~$\omega$ is identified with a value assignment on $\mathfrak{P}$ and a formula $\phi$ with the event given by the worlds $\omega$ with~${\omega \models \phi}$. 

Given the probabilities $\pi (\phi_i) \in (0,1)$ of the formulas $\phi_i$. In general, maximizing the entropy does not yield a distribution that can be easily described using the probabilities~\mbox{$\pi (\phi_i)$, $1 \leq i \leq n$}. As the distribution $\pi$ is essentially determined by one number for every formula $\phi_i$, one aims for a parameterization of~$\pi$ that is also given by one number $w_i \in \mathbb{R}$ for every formula~\mbox{$\phi_i$, $1 \leq i \leq n$}. 

\begin{parametrization}[\cite{Berger}]
Let ${\phi_1, \ldots, \phi_n}$ be propositional formulas. One finds $n$ \textbf{degrees of certainty}, i.e., real numbers~${w_i \in \mathbb{R}}$,~${1\leq i \leq n}$, such that the probability~$\pi(\omega)$ of any world~$\omega$ is given by:
\[
     \pi (\omega) = \left( \prod_{\omega \models \phi_i} \exp(w_i) \right) \cdot \left( \sum_{\omega'~\text{world}} ~~ \prod_{\omega' \models \phi_i} \exp(w_i) \right)^{-1}
\]
\label{parametrization - maximum entropy model}
\end{parametrization}

A LogLinear model of~\cite{MLN} formalizes a set of formulas with degrees of certainty in the sense of Parametrization~\ref{parametrization - maximum entropy model}. 

\begin{definition}[LogLinear Models]
    A \textbf{LogLinear model} is a finite set~$\Phi$ consisting of \textbf{weighted constraints} $(w, \phi)$, where~${w \in \mathbb{R} \cup \{ + \infty, - \infty \}}$ is a weight and $\phi$ is a formula.
    \label{definition - LogLinear model}
\end{definition}

\begin{example}
    The situation in Example~\ref{example - introduction running example} may lead to the LogLinear model:
    $\Phi := \{ (\ln (2),  \textit{cloudy} \rightarrow \textit{rain} ), (\ln (3), \neg \textit{cloudy} \rightarrow \textit{sprinkler}), (+\infty, \textit{wet} \leftrightarrow \textit{rain})  \}$
    \label{example - LogLinear model}
\end{example}

Parametrization \ref{parametrization - maximum entropy model} then yields the following semantics for LogLinear models.

\begin{definition}[Semantics of LogLinear Models]
    Given a LogLinear model $\Phi$, a \textbf{possible world} $\omega$ is a world that models each \textbf{hard constraint} $(\pm \infty,\phi) \in \Phi$, i.e.,~${\omega \models \phi}$ whenever~${(+ \infty, \phi) \in \Phi}$ and ${\omega \models \neg \phi}$ whenever~${(- \infty, \phi) \in \Phi}$. Every possible world~$\omega$ is then associated with the \textbf{weight}
    $$
    w_{\Phi} (\omega) := w (\omega)  := \prod_{\substack{(w, \phi) \in \Phi \\ w \not \in \{ \pm \infty \} \\ \omega \models \phi}} \exp(w) 
    $$
    Set $w (\omega) = 0$ if $\omega$ is not a possible world and  define the \textbf{weight} of a formula $\phi$ to be 
    $
    \displaystyle
    w (\phi) := \sum_{\substack{\omega \text{ world} \\ \omega \models \phi}} w(\omega).
    $
    
    Finally, interpret weights as degrees of certainty and assign to each world or formula the \textbf{probability}
    $
    \pi_{\Phi} (\cdot) := \pi (\cdot) := \dfrac{w(\cdot)}{w(\top)}.
    $
\label{definition - semantics LogLinear model}
\end{definition}

\begin{remark}
    Let $\Phi$ be a LogLinear model. Upon committing to Parametrization~\ref{parametrization - maximum entropy model}, the weighted constraints ${(w, \phi) \in \Phi}$, where $w \in \mathbb{R}$, lack an intuitive interpretation. Only hard constraints $(\pm \infty, \phi) \in \Phi$ enforce that the formula $\phi$ or $\neg \phi$ necessarily holds.  
\end{remark}

\begin{example}
    In the setting of Example~\ref{example - LogLinear model}, one finds that $\pi(\textit{rain} \mid \textit{cloudy}) = \frac{2}{3}$ and ${\pi(\textit{sprinkler} \mid \neg \textit{cloudy}) = \frac{3}{4}}$. Furthermore, it follows that the road is slippery if and only if it is wet.
    \label{example - semantics LogLinear model}
\end{example}

\subsubsection{Bayesian Networks: Causal Relations and Independence}
\label{subsubsec: Causal Bayesian Networks and the Principle of Causal Irrelevance}

Recall from \cite{Causality} how causal relations give rise to conditional independencies:
Identify a \textbf{causal structure} on a set of random variables $\textbf{V}$ with a directed acyclic graph $G$, i.e.~a partial order, on $\textbf{V}$. The intuition is that~${V_1 \in \textbf{V}}$ is a \textbf{cause} of $V_2 \in \textbf{V}$ if there is a directed path from $V_1$ to $V_2$ in $G$. In this case,~$V_2$ is also said to be an \textbf{effect} of $V_1$. Furthermore,~$V_1$ is a \textbf{direct cause} of~$V_2$ if the edge $V_1 \rightarrow V_2$ exists in $G$, i.e. if and only if the node~${V_1 \in \Pa(V_2)}$ lies in the set~$\Pa(V_2)$ of \textbf{direct causes} or \textbf{parents} of~$V_2$. 

\begin{example}
Example~\ref{example - introduction running example} gives rise to the causal structure Graph (\ref{equation - causal structure}) on the Boolean random variables~$
    {\mathfrak{P} := \{ \textit{cloudy},~ \textit{rain},~ \textit{sprinkler},~ \textit{wet},~ \textit{slippery} \}}.
$

Observe that
$\textit{cloudy}$ is a cause of $\textit{slippery}$ but not a direct cause; $\textit{wet}$ is a direct cause of $\textit{slippery}$; there is no causal relationship between $\textit{sprinkler}$ and~$\textit{rain}$.
\label{example - causal structure}
\end{example}

A joint distribution $\pi$ on the random variables $\textbf{V}$ is consistent with a causal structure~$G$ if the influence of any cause $V_1$ on an effect $V_2$ is moderated by the direct causes of $V_2$. \cite{Causality} captures this intuition in the Markov condition:

\begin{definition}[Markov Condition]
A joint distribution~$\pi$ over a set of random variables~$\mathbf{V}$ satisfies the \textbf{Markov condition} with respect to a causal structure~$G$ if every random variable~$V \in \mathbf{V}$ is conditionally independent of all non-effects~${W \not \in \Pa(V)}$ that are not direct causes of~$V$, given its direct causes~$\Pa(V)$.

In this case, the distribution~$\pi$ is said to be \textbf{Markov} to~$G$, written~$\pi \models G$.
\label{definition - Markov condition}
\end{definition}

\begin{example}
In Example \ref{example - causal structure} the Markov condition states for instance that the influence of $\textit{cloudy}$ on the event $\textit{slippery}$ is completely moderated by the event $\textit{wet}$. Once it is known that the pavement of the road is wet, it is expected to be slippery regardless of the event that caused the road to be wet. 
\end{example}

If a distribution $\pi \models G$ satisfies the Markov condition with respect to a given causal structure $G$, it is represented by a Bayesian network on $G$ and vice versa~\mbox{\cite[§1.2.3]{Causality}}:

\begin{definition}[Bayesian Network]
Let $\textbf{V}$ be a finite set of random variables. A \textbf{Bayesian network} $\textbf{BN} := (G, \pi(\cdot \vert \pa(\cdot)))$ on~$\textbf{V}$ consists of a causal structure~$G$ and the probabilities $\pi (v \vert \pa (V)) \in [0,1]$ of the possible values $v$ of the random variables $V \in \textbf{V}$ conditioned on value assignments $\pa(V)$ of their direct causes $\Pa(V)$. 

By applying the chain rule of probability calculus and the Markov condition in Definition \ref{definition - Markov condition}, the Bayesian network $\textbf{BN}$ assigns to  a  value assignment~$\textbf{v}$ on~$\textbf{V}$ the probability:
\begin{align}
    &\pi_{\textbf{BN}} (\textbf{v}) := \pi \left( \textbf{v} \right) :=
\prod_{i=1}^k \pi \left( \textbf{v}(V_i) \vert \pa(V_i) \right), &&\pa(V_i) := \textbf{v} \vert_{\Pa(p_i)},~1 \leq i \leq k
\label{equation - induced distribution of a BN}
\end{align}
\label{definition - Bayesian network}
\end{definition}

\begin{example}
    The causal structure $G$ in Graph (\ref{equation - causal structure}), together with Parameters~(\ref{equation - parameters of a BN}) in Example \ref{example - introduction - sprinkler}, gives rise to a Bayesian network~${\textbf{BN} := (G, \pi(\cdot \vert \pa (\cdot)))}$.
    
    One obtains 
    $
        \pi_{\textbf{BN}} (\textit{cloudy}, \textit{rain}, \textit{sprinkler}, \textit{wet}, \textit{slippery}) = 
        0.5 \cdot 0.6 \cdot 0.1  \cdot 0.9 \cdot 0.8.
    $
    \label{example - Bayesian network}
\end{example}



Fix a Bayesian network~${\textbf{BN} := (G, \pi (\cdot \mid \pa (\cdot)))}$ over a set of random variables~$\textbf{V}$.
\cite{Williamson2001} observes that \emph{maximizing entropy} in Principle~\ref{principle - maximum entropy}, when used to extend the local conditional distributions~$\pi (\cdot \mid \pa (\cdot))$ to a global joint distribution~$\pi'_{\textbf{BN}}$ on~$\textbf{V}$, generally yields a distribution that differs from the Bayesian network semantics~$\pi_{\textbf{BN}}$ in Equation \eqref{equation - induced distribution of a BN}. 

The key insight he provides is that adding a new variable~$W$ to~$\textbf{BN}$ that is not a cause of any other variable leaves the marginal distribution on~$\textbf{V}$ unchanged under Equation~\eqref{equation - induced distribution of a BN}, but may alter it under Principle~\ref{principle - maximum entropy}.

Hence, in general, \emph{maximizing entropy} in Principle~\ref{principle - maximum entropy} and \emph{causal irrelevance} in Principle~\ref{principle - causal irrelevance} are in tension. To address this conflict, \cite{Williamson2001} proposes the following formalization:

\begin{formalization}
    Given a Bayesian network~${\textbf{BN} := (G, \pi(\cdot \mid \pa(\cdot)))}$, an agent who adheres to both \emph{maximizing entropy} in Principle~\ref{principle - maximum entropy} and \emph{causal irrelevance} in Principle~\ref{principle - causal irrelevance} proceeds as follows:

    He begins by maximizing entropy~\( H(\pi) \) subject to the constraint that each source variable~\( V \) with $\Pa(V) = \emptyset$ in~\( G \) takes its possible values~\( v \) with probability~\( \pi(v) \), as specified by~\textbf{BN}. This results in a joint distribution~\( \Tilde{\pi} \) on a subset of variables~\( \textbf{W} \subseteq \textbf{V} \).

    The agent then iteratively maximizes entropy~\( H(\pi) \) under the following constraints, until a full joint distribution~\( \pi \) on~\( \textbf{V} \) is obtained:

    \begin{itemize}
        \item The marginal distribution on~\( \textbf{W} \) is given by~\( \Tilde{\pi} \).
        \item For every variable~\( V \in \textbf{V} \) with direct causes~\( \Pa(V) \subseteq \textbf{W} \), the conditional probabilities~\( \pi(v \mid \pa(V)) \) match the specification in~\textbf{BN}, i.e., \( V \) takes value~\( v \) with probability~\( \pi(v \mid \pa(V)) \) if its parents take the values~\( \pa(V) \).
    \end{itemize}
    \label{formalization - causal irrelevance}
\end{formalization}

In this context, \cite{Williamson2001} obtains the following result:

\begin{theorem}[§5.2, \cite{Williamson2001}]
    Let \( \textbf{BN} := (G, \pi (\cdot \vert \pa(\cdot))) \) be a Bayesian network. The induced distribution \( \pi_{\textbf{BN}}(\cdot) \) is the distribution resulting from the probability specifications \( \pi (\cdot \vert \pa(\cdot)) \) by \emph{maximizing entropy} in Principle~\ref{principle - maximum entropy} and \textit{causal irrelevance} in Principle~\ref{principle - causal irrelevance}, as expressed in Formalization~\ref{formalization - causal irrelevance}.~$\square$ 
    \label{theorem - causal irrelevance and Bayesian networks}
\end{theorem}

\subsubsection{Probabilistic Causal Models}
\label{subsubsec: Probabilistic Functional Causal Models}

\cite{Causality} introduces probabilities into a functional causal model $\mathcal{M}$ by specifying a probability distribution over the situations of $\mathcal{M}$. 

\begin{definition}[Probabilistic Causal Model]
    A \textbf{(probabilistic) (Boolean) causal model} $\mathbb{M} := (\mathcal{M}, \pi)$ consists of a (Boolean) functional causal model~$\mathcal{M}$ together with a probability distribution~$\pi$ on the situations of~$\mathcal{M}$. The \textbf{causal diagram}~$\graph(\mathbb{M})$ of the probabilistic causal model~$\mathbb{M}$ is defined as the causal diagram~$\graph(\mathcal{M})$ of the underlying causal model~$\mathcal{M}$. The model~$\mathbb{M}$ is called \textbf{acyclic} if~$\mathcal{M}$ is acyclic.

    Since~$\mathcal{M}$ is a functional causal model, each situation~$\textbf{u}$ determines a unique solution~$\textbf{s}(\textbf{u})$ of the structural equations. By defining
    \[
        \pi_{\mathbb{M}}(\omega) :=
        \begin{cases}
            \pi(\textbf{u}), & \text{if } \omega = \textbf{s}(\textbf{u}) \\
            0, & \text{otherwise}
        \end{cases}
    \]
    for each value assignment~$\omega$ of the variables~$\textbf{U} \cup \textbf{V}$, the model~$\mathbb{M}$ induces a joint probability distribution on the random variables in~$\textbf{U} \cup \textbf{V}$.
    \label{definition - probabilistic causal model}
\end{definition}

\begin{example}
In Example \ref{example - introduction - sprinkler}, the causal model $\mathbb{M}$ yields a probability distribution $\pi_{\mathbb{M}}$ on the truth value assignments for the variables 
$
\textbf{U} \cup \textbf{V} 
$.  
This allows us, for instance, to calculate the probability $\pi_{\mathbb{M}} (\textit{rain})$ that it rains:
${\pi_{\mathbb{M}} (\textit{rain}) = \pi (u_1) \cdot \pi (u_2) = 0.5 \cdot 0.6 = 0.3
}$
\end{example}

Recall the relation between Bayesian networks and causal models:

\begin{definition}[Markovian Causal Models]
    An acyclic probabilistic causal model $\mathbb{M} := (\mathcal{M}, \pi)$ is \textbf{Markovian} if $\pi$ interprets the error terms as mutually independent random variables. 
    \label{definition - Markovian Causal Models}
\end{definition}

\begin{theorem}[\cite{Causality}, §1.4.2]
    A Markovian causal model~$\mathbb{M}$ gives rise to a distribution~$\pi_{\mathbb{M}}$ that is Markov to its causal diagram, i.e., $\pi \models \graph(\mathbb{M})$. As a result,~$\pi_{\mathbb{M}}$ admits a representation as a Bayesian network over~$\graph(\mathbb{M})$.~$\square$
    \label{theorem - causal models and Bayesian networks}
\end{theorem}

\begin{example}
    The causal model in Example \ref{example - introduction - sprinkler} is Markovian. It gives rise to the Bayesian network, described in Example \ref{example - Bayesian network}.
\end{example}

Again, causal models are not restricted to queries about conditional and unconditional probabilities. They also support queries for intervention effects:

Let~$\mathbb{M} := (\mathcal{M}, \pi)$ be a probabilistic causal model with external variables~$\textbf{U}$ and internal variables~$\textbf{V}$. Given~$\textbf{I} \subseteq \textbf{V}$ and a value assignment~$\textbf{i}$ on~$\textbf{I}$, the \textbf{submodel}~$\mathbb{M}_{\textbf{i}} := (\mathcal{M}_{\textbf{i}}, \pi)$ describes the system under intervention~$\textbf{i}$. The resulting \textbf{post-interventional distribution} is
\[
\pi_{\mathbb{M}}(\cdot \mid \Do(\textbf{i})) := \pi_{\mathbb{M}_{\textbf{i}}}(\cdot).
\]
Here, the \textbf{do-operator}~$\Do(\textbf{i})$ indicates that the variables in~$\textbf{I}$ are fixed by actively doing something. For any event~$A$, the quantity~$\pi(A \mid \Do(\textbf{i}))$ denotes the probability of~$A$ after \textbf{intervening} according to~$\textbf{i}$.

\begin{example}
Recall Example \ref{example - introduction - sprinkler} and ask for the post-interventional probability that the road is slippery after turning off the sprinkler. In this case, one queries the modified model~$\mathbb{M}_{\neg \textit{sprinkler}}$ for $\textit{slippery}$ to obtain the probability
$$
\pi_{\mathbb{M}} (\textit{slippery} \vert \Do (\neg \textit{sprinkler})) = \pi(u_1) \cdot \pi (u_2) \cdot \pi (u_5) \cdot\pi (u_6) =  0.216
$$ 
for the road to be slippery after switching the sprinkler off. 

Note that this result differs from the conditional probability 
$$
\pi_{\mathbb{M}} (\textit{slippery} \vert \neg \textit{sprinkler}) = \dfrac{\pi(u_1) \cdot \pi(u_2) \cdot \pi (\neg u_3) \cdot \pi (u_5) \cdot \pi (u_6)}{\pi(u_1) \cdot \pi (\neg u_3) + \pi(\neg u_1) \cdot \pi (\neg u_4)} = 0.432
$$ that it is slippery if the sprinkler is observed to be off. 
\label{example - intervantions in probabilistic causal models}
\end{example}

Theorem \ref{theorem - causal models and Bayesian networks} yields the following notion of intervention in Bayesian networks:

\begin{definition}[Intervention in Bayesian Networks]
    Let $G := (\textbf{V}, \textbf{E})$ be a directed acyclic graph, and let $\textbf{I} \subseteq \textbf{V}$ be a subset of its nodes. Define the graph
    \[
    G_{\textbf{I}} := (\textbf{V}, \textbf{E}_{\textbf{I}}), \quad \text{where } \textbf{E}_{\textbf{I}} := \{ (V_1, V_2) \in \textbf{E} \mid V_2 \notin \textbf{I} \}
    \]
    by deleting all edges in $G$ that point into nodes in~$\textbf{I}$.

    Let $\textbf{BN} := (G, \pi(\cdot \mid \pa(\cdot)))$ be a Bayesian network inducing the distribution~$\pi$, and let~$\textbf{i}$ be a value assignment on the variables in~$\textbf{I}$. Intervening to force the variables in~$\textbf{I}$ to take the values in~$\textbf{i}$ yields the \textbf{modified Bayesian network} $
    \textbf{BN}_{\textbf{i}} := \left(G_{\textbf{I}}, \pi_{\textbf{i}}(\cdot \mid \pa_{G_{\textbf{I}}}(\cdot))\right)
    $,
    where
    \[
    \pi_{\textbf{i}}(v \mid \pa_{G_{\textbf{I}}}(V)) :=
    \begin{cases}
        1, & \text{if } V \in \textbf{I} \text{ and } v = \textbf{i}(V), \\
        0, & \text{if } V \in \textbf{I} \text{ and } v \neq \textbf{i}(V), \\
        \pi(v \mid \pa(V)), & \text{otherwise}.
    \end{cases}
    \]

    The modified network gives rise to the \textbf{post-interventional distribution}~${\pi_{\textbf{BN}}(\cdot \mid \Do(\textbf{i})) := \pi_{\textbf{BN}_{\textbf{i}}}(\cdot)}$. For any event~$A$, the quantity~$\pi(A \mid \Do(\textbf{i}))$ denotes the probability of~$A$ after \textbf{intervening} according to~$\textbf{i}$.
    \label{definition - intervention in Bayesian networks}
\end{definition}

\begin{example}
    Recall the Bayesian network from Example~\ref{example - Bayesian network}. 
    Intervening and switching the sprinkler on results in the modified Bayesian network $\textbf{BN}_{\textbf{i}}$ with the causal structure $G_I$ that results from the causal structure $G$ in Graph (\ref{equation - causal structure}) by erasing the edge $\textit{cloudy} \rightarrow \textit{sprinkler}$.
    
    The corresponding probabilities are obtained by replacing the conditional probabilities $\pi(\textit{sprinkler} \vert (\neg) \textit{cloudy})$ in Parameters (\ref{equation - parameters of a BN}) with ${\pi(\textit{sprinkler})=1}$, reflecting the intervention.
    \label{example - intervening in a Bayesian network}
\end{example}

Finally, the notion of intervention in causal models aligns with that in Bayesian networks.

\begin{theorem}[\cite{Causality}, §1.4.3]
    Let $\mathbb{M}$ be a Markovian causal model that gives rise to the Bayesian network $\textbf{BN}$ and $\textbf{i}$ a value assignment on a subset of internal variables  $\textbf{I}$ of $\mathbb{M}$. The modified causal model $\mathbb{M}_{\textbf{i}}$ and Bayesian network~$\textbf{BN}_{\textbf{i}}$ induce the same post-interventional distribution~${\pi (\cdot \vert \Do (\textbf{i}))}$.~$\square$
\end{theorem}

\subsection{Causal Systems: A Generic Representation of Causal Reasoning}
\label{subsec: Causal Systems: A Generic Representation of Causal Reasoning}

To reason about \emph{demonstrations} and knowledge-\emph{why} in the presence of uncertainty about \textit{natural necessity} in Principle~\ref{principle - Aquinas}, we propose the notion of a \emph{maximum entropy causal system}:

\begin{definition}[Maximum Entropy Causal System]
    A \textbf{weighted causal rule}~\( (w, R) \) consists of a weight \( w \in \mathbb{R} \cup \{ +\infty, -\infty \} \) and a literal causal rule \( R \). A \textbf{weighted causal theory}~\( \Theta \) is a finite set of weighted causal rules. 

    The \textbf{explanatory content} of a weighted causal theory \( \Theta \) is the \mbox{causal theory}  
    \[
    \explanatory (\Theta) := \{ R \mid \exists w \text{ such that } (w,R) \in \Theta \}.
    \]
    The \textbf{constraint content} of a weighted causal theory \( \Theta \) is the \mbox{LogLinear model}  
    \[
    \constraint (\Theta) := \{ (w, \constraint(R)) \mid (w,R) \in \Theta \}.
    \]

    A \textbf{(maximum entropy) causal system} 
    $
    \textbf{CS} := (\Theta,~\mathcal{E},~\mathcal{O},~\Sigma)
    $ 
    is a tuple consisting of the following components:
    \begin{itemize}
        \item A weighted causal theory~$\Theta(\textbf{CS}) := \Theta$, called \textbf{causal knowledge} of~$\textbf{CS}$.
        \item A set of literals $\mathcal{E}(\textbf{CS}) := \mathcal{E}$, called the \textbf{external premises} of~$\textbf{CS}$.
        \item A set of formulas $\mathcal{O}(\textbf{CS}) := \mathcal{O}$, called the \textbf{observations} of~$\textbf{CS}$.
        \item A \textbf{superordinate science}, that is, a LogLinear model~${\Sigma(\textbf{CS}) := \Sigma}$.
    \end{itemize}
    A \textbf{pure external premise} $\epsilon \in \mathcal{E}$ is a proposition $p \in \mathfrak{P}$ such that $p , \neg p \in \mathcal{E}$. Denote by $\mathcal{E}^{\text{pure}}$ the set of all pure external premises of \textbf{CS}. A \textbf{situation} $\textbf{s}$ of~$\textbf{CS}$ is a value assignment on the pure external premise $\mathcal{E}^{\text{pure}}$.
    
    
    The \textbf{constraint part} $\CONST (\textbf{CS})$ of the causal system $\textbf{CS}$ is the constraint part of its causal knowledge $\Theta$, i.e.,~${
    \CONST (\textbf{CS}) :=  \constraint (\Theta).
    }$ 
    
    The \textbf{explanatory part} $\EXP(\textbf{CS})$ is the deterministic causal system~${
    \EXP (\textbf{CS}) := (\EXP (\Theta), \mathcal{E}, \mathcal{O}).
    }$
    
    
    The system $\textbf{CS}$ is \textbf{without observations} or applies \textbf{default negation} if its explanatory part $\EXP (\textbf{CS})$ does. 
    \label{definition - maximum entropy causal system}
    \label{definition - weighted causal theory}
\end{definition}

We use Definition \ref{definition - weighted causal theory} together with the following guideline:  

\begin{Language}[Maximum Entropy Causal System]
    Fix an area of science as described in Principle~\ref{principle - causal foundation}, and let~$\Delta$, $\mathcal{E}$, and~$\mathcal{O}$ be as in Language~\ref{language - causal systems}. By Assumption~\ref{assumption - literal causal theory}, the causal theory~$\Delta$ consists of literal causal rules~$R_i$, for~${1 \leq i \leq n}$. 
    
    Applying Parametrization~\ref{parametrization - maximum entropy model}, we express degrees of belief in whether \emph{natural necessity}, as described in Principle~\ref{principle - Aquinas}, holds for a rule~$R_i \in \Delta$ by assigning weights~${w_i \in \mathbb{R} \cup \{ \pm \infty \}}$, thereby obtaining a weighted causal theory~$\Theta$.
    
    Assuming that knowledge-\emph{that} from superordinate areas of science is encoded in a LogLinear model~$\Sigma$, we arrive at the maximum entropy causal system~${
        \mathbf{CS} := (\Theta, \mathcal{E}, \mathcal{O}, \Sigma) 
    }$.
    In this setting, \emph{default negation} in Assumption~\ref{assumption - default negation} is satisfied if and only if the system~$\mathbf{CS}$ applies default negation.
    \label{language - maximum entropy causal systems}
    \label{language - Uncertainty in Causal Reasoning}
\end{Language}

We adopt the viewpoint that \emph{maximizing entropy} in Principle~\ref{principle - maximum entropy} corresponds to the \emph{syllogisms} underlying \emph{consistency with deduction} in Principle~\ref{principle - consistency with logical reasoning}.

\begin{formalization}[Syllogisms]
    Recall the situation of Language \ref{language - maximum entropy causal systems}. Adapting Formalizations~\ref{formalization - natural necessity}, the LogLinear model~$\constraint(\Theta)$ encodes beliefs about \emph{natural necessity} in Principle~\ref{principle - Aquinas}.
    
    Let ${\Phi := \constraint (\textbf{CS}) \cup \Sigma}$ be the LogLinear model that corresponds to the causal knowledge~$\Theta$ and the superordinate science $\Sigma$. For every events $A$ and~$B$ there exists a \emph{syllogism} for $B$ with premises $A$ with probability $\pi_{\Phi}(B \mid A)$. In particular, $\pi_{\Phi} (B)$ is the probability that there exists a \emph{syllogism} for $B$ with premises in $\mathcal{E}$. 
    \label{formalization - syllogism}
\end{formalization}

Formalization \ref{formalization - syllogism} motivates the following definition.

\begin{definition}[\emph{That}-Semantics]
    The \textbf{\emph{that}-semantics} of a maximum entropy causal system $\textbf{CS} := (\Theta, \mathcal{E}, \mathcal{O}, \Sigma)$ is the LogLinear model $\Phi$ in Formalization~\ref{formalization - syllogism}. The system assumes \textbf{knowledge-\textit{that}} an event $A$ occurs with probability $\pi_{\textbf{CS}}^{\emph{that}} (A) := \pi_{\Phi} (A \mid \mathcal{O})$.
\end{definition}

According to \textit{directionality} in Principle \ref{principle - directionality}, \mbox{knowledge-\textit{why}} arises from \textit{demonstrations}, i.e., \textit{syllogisms} that follow the causal order of things. We conclude that \emph{causal irrelevance} in Principle \ref{principle - causal irrelevance} applies.  According to Formalization~\ref{formalization - causal irrelevance} of \mbox{\cite{Williamson2001}}, this means that the \emph{entropy} in Principle \ref{principle - maximum entropy} needs to be \emph{maximized} greedily along the given causal order:

Recall that two nodes \( p \) and \( q \) of a directed graph \( G := (V,E) \) are 
\textbf{strongly connected}, written~\( p \sim q \), if there exist directed paths 
from \( p \) to \( q \) and from~\( q \) to \( p \) in~\( G \). Strong connectedness 
(\(\sim\))/2 is an equivalence relation, and the equivalence classes \( [p] \in V/\sim \) 
are called the \textbf{strongly connected components} of \( G \). Lastly, the resulting 
\textbf{factor graph} \( G/\sim := (V/\sim, E/\sim) \) is a directed acyclic graph, 
where \( E/\sim := \{ ([p],[q]) \in (V/\sim)^2 \mid (p,q) \in E \} \).

Let $\textbf{CS} := (\Theta, \mathcal{E}, \mathcal{O}, \Sigma)$ be a maximum entropy causal system, and denote by~${
G(\textbf{CS}) := \graph(\explanatory(\textbf{CS})) := (\textbf{V}, \textbf{E})
}$ 
the factor graph on the strongly connected components of the causal structure of the explanatory content of $\textbf{CS}$. To each component \(V \in \textbf{V}\), associate a random variable ranging over all assignments
${
v : V \to \{\textit{True}, \textit{False}\}.
}$

Let \( v \) be a value of a component \( V \in \textbf{V} \) such that \( V \cap \mathcal{E}^{\mathrm{pure}} = \emptyset \), and~\( \pa(V) \) a value assignment to \( \Pa(V) \). Define
${
    \Theta \vert V := \{ (w, \phi \Rightarrow l) \in \Theta \mid l \in \{p, \neg p\},~ p \in V \}
}$ and ${
    \mathcal{E}^{\Pa(V)} := \{ l \in \{p, \neg p\} \mid p \in W,~ W \in \Pa(V) \}
}$.

Consider the system \( \textbf{CS}^{V} := (\Theta \vert V, \mathcal{E}^{\Pa(V)}, \emptyset) \). As all propositions \( p \in V \) lie in a causal cycle, we argue that \emph{causal irrelevance}, as stated in Principle~\ref{principle - causal irrelevance}, does not yield any constraint. Hence, applying \emph{sufficient causation} in Assumption~\ref{principle - Leibniz} as formalized in Formalization~\ref{formalization - sufficient causation}, and Formalization~\ref{formalization - syllogism}, we conclude that there exists a \emph{demonstration} of \( v \) with premises \( \pa(V) \) with probability
\begin{equation}
    \pi_{\textbf{CS}}^{\emph{why}}(\textbf{v} \mid \pa(V)) := \pi_{\constraint(\Theta \vert V)}(\textbf{v} \mid \pa(V), \sufficient(\textbf{CS}^{V})).
\label{equation - demonstration}
\end{equation}

We proceed by applying \emph{causal irrelevance} from Principle~\ref{principle - causal irrelevance} and \emph{maximizing entropy} from Principle~\ref{principle - maximum entropy}, as formalized in Formalization~\ref{formalization - causal irrelevance} together with Theorem~\ref{theorem - causal irrelevance and Bayesian networks}, to obtain the \emph{causal semantics} for maximum entropy causal systems:

\begin{definition}[Causal Semantics]
     Let~$\textbf{CS}$ be a maximum entropy causal system with~$G(\textbf{CS}) := (\textbf{V}, \textbf{E})$. We say that~$\textbf{CS}$ \textbf{provides demonstrations} if the following conditions hold:
    \begin{itemize}
        \item[i)] For every component~${V \in \textbf{V}}$ and every assignment~${\pa(V)}$, the probabilities in Equation~\eqref{equation - demonstration} sum to one.
        \item[ii)] There is no weighted rule~${(w, \phi \Rightarrow (\neg) p)}$ with~${p \in \mathcal{E}^{\textbf{pure}}}$.
        \item[iii)] For every proposition~$p \in \mathfrak{P}$ there exists a rule ${(w, \phi \Rightarrow (\neg)p) \in \Theta}$ or an external premises~${(\neg)p \in \mathcal{E}}$.
    \end{itemize}

    In this case, the \textbf{causal structure} $\graph (\textbf{CS})$ of $\textbf{CS}$ results from $G(\textbf{CS})$ by replacing all nodes $V := \{ p \}$ for $p \in \mathcal{E}^{\text{pure}}$ and all outgoing edges $V \rightarrow W$ with one node $S := \mathcal{E}^{\text{pure}}$ and the edges $S \rightarrow W$.  
    
    The system $\textbf{CS}$ then assumes the \textbf{indemonstrable knowledge} that every situation $\textbf{s}$ occurs with probability: 
    \begin{equation}
     \pi_{\textbf{CS}}^{\emph{why}}(\textbf{s}) := 
     \pi_{\Sigma} (\textbf{s}).     
    \label{equation - indemonstrable knowledge}
    \end{equation}

    Let $\textbf{BN}(\textbf{CS})$ be the Bayesian network that is given by the causal structure $\graph (\textbf{CS})$ and the Probabilities (\ref{equation - demonstration}) and (\ref{equation - indemonstrable knowledge}). The \textbf{causal semantics} of $\textbf{CS}$
    is then given by setting
    $
    \pi_{\textbf{CS}}^{\text{causal}} (\omega) := \pi_{\textbf{BN}(\textbf{CS})}(\omega \mid \mathcal{O})
    $ for each world $\omega$.

    Finally, denote by $\mathcal{O} \vert \mathcal{E}$ the set of all observations $o \in \mathcal{O}$ that can be formed with the external premises $\mathcal{E}$ and set $\textbf{CS} \vert \mathcal{E} := (\Theta, \mathcal{E}, \mathcal{O}\vert \mathcal{E}, \Sigma)$. 
    
    If $\pi_{\textbf{CS}}^{\emph{causal}}(A) = \pi_{\textbf{CS}\vert\mathcal{E}}^{\emph{causal}}(A)$ for an event \( A \), a maximum entropy causal system~$\textbf{CS}$ that provides demonstrations assumes \mbox{\textbf{knowledge-\emph{why}}}~\( A \) occurs with probability~\({
         {\pi_{\textbf{CS}}^{\emph{why}}(A) := \pi_{\textbf{CS}}^{\text{causal}}(A)}.
    }\)
    \label{definition - causal semantics}
\end{definition}

To summarize, we argue for the following result:

\begin{formalization}
    In Language~\ref{language - maximum entropy causal systems}, for a causal system~$\textbf{CS}$ that provides demonstrations, the causal semantics~$\pi_{\textbf{CS}}^{\text{causal}}$ is the distribution that results from the combination of: 
    Principle~\ref{principle - consistency with logical reasoning}, the notion of \emph{syllogism} in Formalization~\ref{formalization - syllogism}, \emph{sufficient causation} in Assumption~\ref{principle - Leibniz}, \emph{causal irrelevance} in Principle~\ref{principle - causal irrelevance}, and \emph{maximizing entropy} in Principle~\ref{principle - maximum entropy}, as specified in Formalization~\ref{formalization - causal irrelevance}.
    
    Finally, Principle~\ref{principle - directionality} implies that the system~$\textbf{CS}$ possesses knowledge-\emph{why} as characterized in Definition~\ref{definition - causal semantics}.
    \label{formalization - knowledge-why under uncertainty}
\end{formalization}

It remains to show that the causal semantics indeed induces a well-defined probability distribution.
    
\begin{proposition}
    Let $\textbf{CS}$ be a maximum entropy causal system that provides demonstrations.  
    Then the causal semantics $\pi_{\textbf{CS}}^{\text{causal}}$ induces a probability distribution on the worlds $\omega$ of $\mathfrak{P}$.
    \label{proposition - causal semantics is well-defined}
\end{proposition}

\begin{proof}
    The causal structure ${\graph (\textbf{CS}) = (\textbf{V}, \textbf{E})}$ is a directed acyclic graph.  
    By Assertion~i) of Definition~\ref{definition - causal semantics}, $\textbf{BN}(\textbf{CS})$ is therefore a Bayesian network as in Definition~\ref{definition - Bayesian network}, with nodes $V \in \textbf{V}$ interpreted as random variables taking truth value assignments $v:V \to \{\top,\bot\}$.  
    Equation~(\ref{equation - induced distribution of a BN}) thus defines a joint distribution on~$\textbf{V}$.

    By Assertions~ii) and iii) of Definition~\ref{definition - causal semantics}, $\textbf{V}$ forms a partition of $\mathfrak{P}$.  
    Hence, each world $\omega : \mathfrak{P} \to \{\top,\bot\}$ corresponds uniquely to the tuple of restrictions~$(\omega|_V)_{V \in \textbf{V}} \in \prod_{V \in \textbf{V}} V^{\{\top,\bot\}}$.  
    Consequently, $\textbf{BN}(\textbf{CS})$ induces a unique probability distribution over worlds $\omega$, which is precisely $\pi_{\textbf{CS}}^{\text{causal}}$.
\end{proof}

\subsection{Interpreting Current Artificial Intelligence Technologies as Causal Systems}
\label{subsec: Interpreting Current Artificial Intelligence Technologies as Causal Systems}

To demonstrate the effectiveness of the proposed approach, we interpret the widely used formalisms from Section~\ref{subsec: Preliminaries for Uncertainty} as instances of maximum entropy causal systems. This allows us to analyze the resulting forms of knowledge and to extend the treatment to external interventions.

\subsubsection{Pearl's Probabilistic Causal Models and Interventions in Causal Systems}
\label{subsubsec: Pearl's Probabilistic Causal Models}

Maximum entropy causal systems without observations, which apply default negation, can express the probabilistic causal models of \mbox{\cite{Causality}}:

\begin{definition}[Bochman Transformation]
    The \textbf{Bochman transformation} of a probabilistic causal model $\mathbb{M} := (\mathcal{M}, \pi)$ with $\mathcal{M} :=  (\textbf{U}, \textbf{V}, \Error, \Pa, \textbf{F})$ is the causal system~${\textbf{CS}(\mathbb{M}) := (\Theta, \mathcal{E}, \emptyset, \Sigma)}$, defined by ${\Theta := \{ (+ \infty, F_V \Rightarrow V) \}_{ V \in \textbf{V} }}$; 
    $\mathcal{E} := \textbf{U} \cup \{ \neg V \mid V \in \textbf{U} \cup \textbf{V} \}$; $
     \Sigma := \{ (\ln(\pi(\textbf{s})),~\bigwedge \textbf{s}) \text{: } \textbf{s}~\text{situation of } \mathcal{M} \}
    $
    
    Here, we identify a situation $\textbf{s}$ of $\mathcal{M}$  with a set of literals. 
    

    \label{definition - Bochman transformation of probabilistic causal models}
\end{definition}

\begin{example}
   In Example~\ref{example - introduction - sprinkler}, the Bochman transformation of the causal model $\mathbb{M}$ is given by the maximum entropy causal system $\textbf{CS} (\mathbb{M})$. 
\label{example - Bochman transformation of probabilistic causal models}
\end{example}

The Bochman transformation~$\textbf{CS} (\mathbb{M})$ of an acyclic causal model $\mathbb{M}$ possesses knowledge-\textit{why} $\pi_{\textbf{CS}(\mathbb{M})}^{\textit{why}}$ that corresponds to the distribution $\pi_{\mathbb{M}}$ associated with~$\mathbb{M}$.

\begin{theorem}
    Let $\mathbb{M}$ be an acyclic probabilistic causal model with Bochman transformation $\textbf{CS}(\mathbb{M})$. For every formula $\phi$, the causal system $\textbf{CS}(\mathbb{M})$ assumes the knowledge-\emph{why} 
    $
    \pi_{\textbf{CS}}^{\textit{why}} (\phi) = \pi_{\mathbb{M}}(\phi)
    $. 
    \label{theorem - Bochman transformation for deterministic causal models}
\end{theorem}

\begin{proof}
    Observe that the system $\textbf{CS}(\mathbb{M})$ provides demonstrations by construction. Since it is without observations, it assumes knowledge-\textit{why} given by its causal semantics. We proceed by induction on the number $n$ of internal variables. 
    
    If $n=1$, by Theorems \ref{theorem - completion of a causal theory} and \ref{theorem - causal models and Bayesian networks}, every world $\omega$ with $\pi_{\mathbb{M}}(\omega) > 0$ is a causal world of $\explanatory(\textbf{CS}(\mathbb{M}))$.  
    Finally, by Theorem \ref{proposition - splitting in logical and explanatory content under causal sufficiency}, we conclude that  
    $
        \pi (\sufficient (\textbf{CS}(\mathbb{M}))) = 1
    $ to obtain the desired result.

    Assume that $n > 1$ and choose a sink $V$ in $\graph(\mathbb{M})$ and let $\mathbb{M} \setminus V$ denote the model that results from $\mathbb{M}$ by erasing the structural equation for $V$. According the inductions hypothesis the induced distribution $\pi_{\mathbb{M} \setminus V}$ and the causal semantics $\pi_{\textbf{CS}(\mathbb{M} \setminus V)}$ coincide. Now argue analogously to the case $n=1$ to obtain the desired result.
\end{proof}


Analogously to Section~\ref{subsubsec: External Interventions in Causal Systems}, we introduce the following notion of intervention in maximum entropy causal systems.

\begin{definition}[Modified Causal Systems]
    Let $\textbf{CS} := (\Theta, \mathcal{E}, \mathcal{O}, \Sigma)$ be a causal system, and let $\textbf{i}$ be a value assignment on a set of atoms $\textbf{I} \subseteq \mathfrak{P}$. To represent the intervention of forcing the atoms in $\textbf{I}$ to attain values according to the assignment $\textbf{i}$, we construct the \textbf{modified causal system}~${
    \textbf{CS}_{\textbf{i}} := (\Theta_{\textbf{i}}, \mathcal{E}_{\textbf{i}}, \mathcal{O}, \Sigma),
    }$
    which is obtained from $\textbf{CS}$ by applying the following modifications:
    \begin{itemize}
        \item Remove all rules $(w, \phi \Rightarrow p) \in \Theta$ and $(w, \phi \Rightarrow \neg p) \in \Theta$ for all~${p \in \textbf{I}}$.
        \item Remove all external premises $p \in \mathcal{E}$ and $\neg p \in \mathcal{E}$ for all~${p \in \textbf{I}}$.
        \item Add a  weighted rule $(+\infty, \top \Rightarrow l)$ to $\Theta_{\textbf{i}}$ for all literals ${l \in \textbf{I}}$.
    \end{itemize} 
    \label{definition - modified maximum entropy causal systems}
\end{definition}

    

Once again, for acyclic probabilistic causal models, the concept of intervention defined in Definition~\ref{definition - modified maximum entropy causal systems} is consistent with the Bochman transformation in Definition~\ref{definition - Bochman transformation of probabilistic causal models}.

\begin{proposition}
    Let $\mathbb{M} := (\mathcal{M}, \pi)$ be an acyclic probabilistic causal model and let $\textbf{i}$ be a truth value assignment on the internal variables ${\textbf{I} \subseteq \textbf{V}}$. 

    The causal systems~${\textbf{CS}(\mathbb{M}_{\textbf{i}})}$ and $\textbf{CS} (\mathbb{M})_{\textbf{i}}$ assume the same knowledge-\emph{why}, i.e.,~${
    \pi_{\textbf{CS}(\mathbb{M}_{\textbf{i}})}^{\textit{why}} (\omega) = \pi_{\textbf{CS} (\mathbb{M})_{\textbf{i}}}^{\textit{why}} (\omega)
    }$ for every world  $\omega$.
    \label{proposition - maximum entropy Bochman transformation is well-defined}
\end{proposition}

\begin{proof}
    Since $\mathbb{M}$ is acyclic, we may without loss of generality assume that it only consists of one structural equation and conclude as in Theorem \ref{theorem - Bochman transformation for deterministic causal models}.
    
    Let $\omega$ be a world. According to Proposition \ref{proposition - deterministic Bochman transformation is well-defined}, the deterministic causal systems $\explanatory(\textbf{CS}(\mathbb{M}_{\textbf{i}}))$ and $\explanatory(\textbf{CS}(\mathbb{M})_{\textbf{i}}) = \explanatory(\textbf{CS}(\mathbb{M}))_{\textbf{i}}$ have the same causal worlds. Since all rules in the causal systems under consideration have weight $+\infty$, Proposition \ref{proposition - splitting in logical and explanatory content under causal sufficiency} implies that $\pi_{\textbf{CS}(\mathbb{M}_{\textbf{i}})}^{\textit{why}} (\omega) > 0$ if and only if $\pi_{\textbf{CS} (\mathbb{M})_{\textbf{i}}}^{\textit{why}} (\omega) > 0$, and in that case, $\omega$ is a causal world of the aforementioned deterministic causal systems. In particular, the probability of $\omega$ is uniquely determined by the corresponding situation, which is calculated from the same LogLinear model in all causal systems under consideration.
\end{proof}

\emph{Non-interference} in Principle \ref{principle - non-interference} then motivates the following definition:

\begin{definition}[Semantics of External Interventions]
    Let $\textbf{CS}$ be a maximum entropy causal system, and let $\textbf{i}$ be a value assignment on a set of atoms~${\textbf{I} \subseteq \mathfrak{P}}$, leading to the modified causal system~$\textbf{CS}_{\textbf{i}}$. The system~$\textbf{CS}$ \textbf{assumes} that a formula~$\phi$ is true \textbf{after intervention}~$\textbf{i}$ with probability~$\pi_{\textbf{CS}}(\phi \mid \Do(\textbf{i})) \in [0,1]$ if and only if
    \(
    \pi_{\textbf{CS}}(\phi \mid \Do(\textbf{i})) = \pi_{\textbf{CS}_{\textbf{i}}}^{\mathit{why}}(\phi)
    \)
    and the distribution induced by~$\Sigma(\textbf{CS})$ renders~$\textbf{I} \cap \mathcal{E}^{\text{pure}}$ and~$\mathcal{E}^{\text{pure}} \setminus (\textbf{I} \cap \mathcal{E}^{\text{pure}})$ independent.

    \label{definition - semantics of external interventions in maximum entroy causal systems}
\end{definition}  

In summary, we argue for the following result.

\begin{formalization}
    Let us fix an area of science such that Languange~\ref{language - maximum entropy causal systems} yields maximum entropy causal system $\textbf{CS}$ that provides demonstrations. Under these conditions and Principle \ref{principle - non-interference}, Definitions~\ref{definition - modified maximum entropy causal systems} and~\ref{definition - semantics of external interventions in maximum entroy causal systems} correctly characterize the knowledge represented by~$\textbf{CS}$ regarding the effects of external interventions.  
    \label{formalization - external interventions deterministic}
\end{formalization}

\subsubsection{LogLinear Models}
\label{subsubsec: LogLinear Models}
We define the \textbf{Bochman interpretation} of a LogLinear model $\Phi$ as the maximum entropy causal system 
${
\textbf{CS}(\Phi) := (\emptyset, \emptyset, \emptyset, \Phi). 
}$

It follows that $\pi_{\textbf{CS}(\Phi)}^{\textit{that}}(\omega) = \pi_{\Phi} (\omega)$ for all worlds $\omega$. Since it does not provide demonstrations, the causal system ${\textbf{CS}(\Phi) := (\emptyset, \emptyset, \emptyset, \Phi)}$ does not possess knowledge-\textit{why}. 

Now, assume that we intervene according to a truth value assignment $\textbf{i}$ on the atoms $\textbf{I} \subseteq \mathfrak{P}$, yielding the modified system
${
\textbf{CS}(\Phi)_{\textbf{i}} := (\{ \top \Rightarrow l \mid l \in \textbf{i} \}, \emptyset, \emptyset, \Phi). 
}$
Unless $\textbf{i}$ represents a world with $\textbf{I} = \mathfrak{P}$, we find that $\textbf{CS}_{\textbf{i}}$ does not provide demonstrations and $\textbf{CS}(\Phi)$ lacks knowledge about intervention effects. 

Interpreting every probability distribution as a LogLinear model, we conclude that probability distributions neither possess knowledge-\textit{why} nor knowledge about external interventions.

\subsubsection{Bayesian Networks}
\label{subsubsec: Bayesian Networks}
The \textbf{sigmoid function} $ \sigma : \mathbb{R} \cup \{ \pm \infty \}  \rightarrow [0,1]$, $\displaystyle w \mapsto 
       \begin{cases}
       \frac{\exp(w)}{1 + \exp(w)} , & w \in \mathbb{R},  \\
       0, & w = - \infty, \\
       1, & w = + \infty
       \end{cases}
$
is~\mbox{bijective}, and we write $\sigma^{-1} : [0,1] \rightarrow \mathbb{R}$ for its inverse.

Let~$\textbf{BN} := (G, \pi(\cdot, \pa(\cdot)))$ be a Boolean Bayesian network. The \textbf{Bochman transformation} of~$\textbf{BN}$ is the causal system 
$
\textbf{CS}(\textbf{BN}) := (\Theta, \mathcal{E}, \emptyset, \Sigma)
$,
defined as follows:

\begin{itemize}
    \item The weighted causal theory $\Theta$ consists of the  rules
    ${
    (w, \pa(p) \Rightarrow p)
    }$
    for every non-source node $p$ in $G$ and every truth value assignment $\pa(p)$ of its direct causes $\Pa(p) \neq \emptyset$, where 
    \begin{equation}
        w := \sigma^{-1} \left[~ \pi_{\textbf{BN}}(p \vert \pa(p)) \cdot \pi_{\textbf{BN}}(\pa(p)) + \pi_{\textbf{BN}} (\neg \pa(p)) ~\right].
        \label{equation - parameters in Bochman transformation}
    \end{equation}
    \item The external premises are given by
    $
    \mathcal{E} := \textbf{S} \cup \{ \neg p \mid p \in \mathfrak{P} \},
    $
    where $\textbf{S}$ denotes the set of source nodes $S$ with $\Pa(S) = \emptyset$ in the graph~$G$.
    \item The superordinate science is given by
    $
    \Sigma := \{ (\sigma^{-1}(\pi (s)),~ s) \mid s \in \textbf{S} \}.
    $
\end{itemize}

The Bochman transformation~$\textbf{CS} (\textbf{BN})$ of a Bayesian network $\textbf{BN}$ possesses knowledge-\textit{why} $\pi_{\textbf{CS}(\textbf{BN})}^{\textit{why}}$ that corresponds to the  associated distribution $\pi_{\textbf{BN}}$.

\begin{theorem}
    Let $\textbf{BN} := (G, \pi (\cdot \vert \pa (\cdot)))$ be a Boolean Bayesian network with Bochman transformation $\textbf{CS} := \textbf{CS}(\textbf{BN}) := (\Theta, \mathcal{E}, \mathcal{O}, \Sigma)$. 
    
    The system $\textbf{CS}$ possesses \mbox{knowledge-\textit{why}}  
    $
    \pi_{\textbf{CS}}^{\emph{why}} (\omega) = \pi_{\textbf{BN}} (\omega)
    $
    for all worlds~$\omega$.    
    \label{theorem - Bochman transformation of Bayesian networks}
\end{theorem}

\begin{proof}
    For every variable $V := \{ p \}$ with $\Pa (V) \neq \emptyset$ in ${\graph (\textbf{CS})}$ of $\textbf{CS}$ we find 
    ${\pi_{\constraint(\Theta \vert V)} (p \vert \pa (p), \sufficient(\textbf{CS}^V))   \stackrel{\text{construction}}{=} \pi_{\textbf{BN}} (p \vert \pa (p))}$.  
    Furthermore, we find $\pi_{\textbf{CS}}^{\emph{why}}(\textbf{s}) = \pi_{\textbf{BN}} (\textbf{s})$ for every situation $\textbf{s}$ of $\textbf{CS}$, that is, for every value assignment on all variables $S$ with $\Pa(S) = \emptyset$ in $G$. Hence, the desired result follows. 
\end{proof}

\begin{example}
    In Example~\ref{example - introduction - sprinkler}, the causal system~$\textbf{CS}(\textbf{BN})$  is the Bochman transformation of the Bayesian network $\textbf{BN}$ if we choose Parameters (\ref{equation - parameters in Bochman transformation}).
\end{example}

Now, consider a Boolean Bayesian network $\textbf{BN} := (G,\pi(\cdot,\pa(\cdot)))$ on the variables $\mathfrak{P}$, and let $\textbf{i}$ be a truth value assignment on a subset of variables~${\textbf{I} \subseteq \mathfrak{P}}$. Intervene according to $\textbf{i}$ to obtain the Bayesian network ${\textbf{BN}_{\textbf{i}} := (G_{\textbf{I}},\pi_{\textbf{i}}(\cdot,\pa(\cdot)))}$ and the causal system
$
\textbf{CS}_{\textbf{i}} := \textbf{CS}(\textbf{BN})_{\textbf{i}} := (\Theta_{\textbf{i}}, \mathcal{E}_{\textbf{i}}, \emptyset, \Sigma).
$

By definition, the distribution $\pi_{\textbf{CS}_{\textbf{i}}}^{\textit{why}}$ is Markov to the graph $\textbf{G}_{\textbf{I}}$. As in Theorem~\ref{theorem - Bochman transformation of Bayesian networks}, we can verify that 
$
\pi_{\textbf{CS}}^{\textit{why}}(p \mid \pa (p)) = \pi_{\textbf{BN}_{\textbf{i}}}(p \mid \pa (p))
$ 
for all~${p \in \mathfrak{P}}$. We conclude that the Bayesian network $\textbf{BN}$ and its Bochman transformation~$\textbf{CS}(\textbf{BN})$ predict the same effects of external interventions:

\begin{theorem}
    Let $\textbf{BN} := (G, \pi (\cdot \vert \pa (\cdot)))$ be a Boolean Bayesian network with Bochman transformation $\textbf{CS} := \textbf{CS}(\textbf{BN}) := (\Theta, \mathcal{E}, \mathcal{O}, \Sigma)$. 
    
    The system $\textbf{CS}$ assumes that an event $A$ is true after intervening according to an assignment $\textbf{i}$ on $\textbf{I}$ with probability 
    $
    \pi_{\textbf{CS}} (A \mid \Do (\textbf{i})) = \pi_{\textbf{BN}} (A \mid \Do (\textbf{i}))
    $.~$\square$
\end{theorem}

\section{Conclusion}
\label{sec: Conclusion}
This paper introduces causal systems as a formal framework for distinguishing between knowledge-\textit{that} and \mbox{knowledge-\textit{why}}, as defined in Aristotle's \textit{Posterior Analytics}. It argues that external interventions can be meaningfully treated only on the basis of knowledge-\textit{why}. Embedding existing artificial intelligence technologies into the formalism of causal systems enables a classification of the type of knowledge they provide and an assessment of the feasibility of handling external interventions.

This work embeds LogLinear models \citep{MLN}, as well as Bayesian networks and causal models \citep{Causality}, into the framework of causal systems. \cite{rueckschloß2025rulesrepresentcausalknowledge} interpret abductive logic programs \citep{AbductiveLogicPrograms} as deterministic causal systems. In future work, it is further envisaged to analyze ProbLog programs \citep{ProbLog, Problog1} and $\LP^{\MLN}$ programs \citep{LPMLN} as maximum entropy causal systems.

We further propose extending maximum entropy causal systems to the context of first-order logic. We conjecture that the resulting theory will be expressive enough to encompass probabilistic logic programming \citep{PLP}, Markov logic networks \citep{MLN}, and relational Bayesian networks \citep{Jaeger97}. This would establish a unifying framework for \textit{relational artificial intelligence} \citep{StarAI}, interpreting it as the study of formalisms that capture the fundamental concepts of symmetry, uncertainty, and causal explanation.

According to \cite{Causality}, causal models can answer counterfactual queries, whereas Bayesian networks cannot. As a direction for future research, we propose characterizing the additional knowledge captured in causal models that enables this type of query.

\bibliographystyle{elsarticle-harv} 
\bibliography{literature}

\newpage

\appendix

\section{Glossary of Technical Terms and References}
\label{section: Glossary}

\begin{footnotesize}
\begin{longtable}{@{}p{4.2cm}p{11.8cm}@{}}
\toprule
\textbf{Term / Reference} & \textbf{Description} \\
\midrule
\endfirsthead
\toprule
\textbf{Term / Reference} & \textbf{Description} \\
\midrule
\endhead
$(\Rrightarrow)/2$ & Expresses explainability (e.g. Production Inference Relation) \\
Exact Theories & Set of formulas satisfying Principle 6 and Assumption 7 \\
Causal Worlds & Worlds satisfying Principle 6 and Assumption 7 \\
Causal Rules $\phi \Rightarrow \psi$ & Represent causal knowledge \\
Literal Causal Rules $\phi \Rightarrow l$ & Causal rules with literal effect $l$ \\
Atomic Causal Rules $\phi \Rightarrow p$ & Causal rules with atomic effect $p$ \\
Causal Theory $\Delta$ & Set of causal rules \\
Deterministic Causal System $\textbf{CS} := (\Delta, \mathcal{E}, \mathcal{O})$ & $\Delta$: literal causal theory, $\mathcal{E}$: external premises, $\mathcal{O}$: observations \\
Weighted Causal Rules & Weighted rule $(w, \phi \Rightarrow \psi)$ with $w \in \mathbb{R} \cup \{ \pm \infty \}$ \\
Maximum Entropy Causal System $\textbf{CS} := (\Theta, \mathcal{E}, \mathcal{O}, \Sigma)$ & $\Theta$: weighted rules, $\mathcal{E}$: external premises, $\mathcal{O}$: observations, $\Sigma$: superordinate science  \\
Bochman Transformation & Translation into the framework of causal systems \\
Modified \dots & Realization of interventions \\

Principle 1 & Causal explanations are instances of logical deduction (Consistency with Deduction) \\
Principle 2 & Causal explanations follow the causal order (Directionality) \\
Principle 3 & Causal explanations are finite and acyclic (Non-Circularity)\\
Principle 4  & Causal explanations root in external premises $\mathcal{E}$ (Causal Foundation)\\
Principle 5  & Causal knowledge is stated in rules (Causal Rules)\\
Principle 6  & Everything that is explainable actually holds (Natural Necessity)\\
Principle 8  & Intervention effects propagate along the causal direction (Non-Interference)\\
Principle 9  & Unobserved non-causes do not matter (Causal Irrelevance)\\
Principle 27  & Extend beliefs by maximizing entropy $H(\pi)$ (Maximum Entropy)\\

Assumption 7  & Everything that holds is explainable (Sufficient Causation)\\
Assumption 17  & Negative literals do not require an explanation (Default Negation) \\
Assumption 20 & Causal theories are literal \\

Language 10 & Propositional logic represents reasoning about knowledge-\emph{that} \\
Language 11 & Explainability $(\Rrightarrow)/2$ is a binary relation on formulas \\
Language 15 & Relates explainability $(\Rrightarrow)/2$ to the acquisition of knowledge-\emph{why} \\
Language 16 & Principle 6 and Assumption 7 lead to determinate causal theories \\
Language 18 & Assumption 17 gives rise to causal theories with default negation \\
Language 19 & Deterministic causal systems capture areas of science in deterministic case\\
Parametrization 28 & Suitable parametrization for probability spaces in combination with Principle 27 \\
Language 30 & Maximum entropy causal systems capture areas of science under uncertainty\\

Formalization 12 & Meaning of Principle 1 in the deterministic case \\
Formalization 13 & Meaning of Principle 6 and Assumption 7 in the deterministic case \\
Formalization 14 & Meaning of Principle 4 \\
Formalization 18 & Meaning of Assumption 17 \\
Formalization 21 & Meaning of areas of science in the deterministic case \\
Formalization 22 & Meaning of Principle 9 in the deterministic case \\
Formalization 23 & Meaning of knowledge-\emph{why} in deterministic case \\
Formalization 24 & Meaning of Principle 8 in deterministic case \\
Formalization 25 & Meaning of Principle 6 in deterministic case \\
Formalization 26 & Meaning of Assumption 7 in deterministic case \\
Formalization 29 & Resolves conflict between Principle 9 and Principle 27 \\
Formalization 31 & Relates Principle 1 and Principle 27 \\
Formalization 32 & Meaning of Principles 1,2,4,9,27, Assumption 7 and knowledge-\emph{why} under uncertainty \\
Formalization 33 & Meaning of Principle 8 under uncertainty \\
\bottomrule
\end{longtable}
\end{footnotesize}




\end{document}